\documentclass{article}

\PassOptionsToPackage{numbers}{natbib}
\usepackage[preprint]{neurips_2024}

\usepackage[utf8]{inputenc} %
\usepackage[T1]{fontenc}    %
\usepackage{url}            %
\usepackage{booktabs}       %
\usepackage{amsfonts}       %
\usepackage{nicefrac}       %
\usepackage{microtype}      %
\usepackage{xcolor}         %
\usepackage[inkscapelatex=false]{svg}
\usepackage{wrapfig}

\usepackage{amsmath}
\usepackage{amssymb}
\DeclareMathOperator*{\argmin}{argmin}

\usepackage{mathtools}
\usepackage{comment}
\usepackage{bm}
\usepackage{enumitem}

\usepackage[
  colorlinks = true,
  citecolor  = black,
  linkcolor  = black,
  urlcolor   = black,
  unicode,
]{hyperref}

\usepackage{hyperref}
\usepackage[capitalise]{cleveref}

\global\long\def\bX{\mathbf{X}}%
\global\long\def\bY{\mathbf{Y}}%
\global\long\def\1{\mathbf{1}_{N}}%
\global\long\def\NTK{\mathsf{NTK}}%
\global\long\def\dt{\frac{d}{dt}}%

\usepackage{amsthm}
\usepackage{thmtools,thm-restate}

\declaretheorem[name=Corollary]{corollary}

\title{Early learning of the optimal constant solution in neural networks and humans}

\author{
Jirko Rubruck\,$^1$%
\\
  University of Oxford
  \And
  Jan P. Bauer\,$^2$ \\
  The Hebrew University
  \AND
  Andrew Saxe\,$^\dag$\\
  University College London
  \And
  Christopher Summerfield\,$^\dag$\\
  University of Oxford 
}

\newcommand\blfootnote[1]{%
  \begingroup
  \renewcommand\thefootnote{}\footnote{#1}%
  \addtocounter{footnote}{-1}%
  \endgroup
}

\begin{document}
\bibliographystyle{unsrtnat}

\maketitle
\blfootnote{$^1$\,\texttt{jirko.rubruck@psy.ox.ac.uk}, $^2$\,\texttt{jan.bauer@mail.de}}
\blfootnote{$^\dag$\,Co-senior authors}

\begin{abstract}

Deep neural networks learn increasingly complex functions over the course of training. Here, we show both empirically and theoretically that learning of the target function is preceded by an early phase in which networks learn the optimal constant solution (OCS) – that is, initial model responses mirror the distribution of target labels, while entirely ignoring information provided in the input. Using a hierarchical category learning task, we derive exact solutions for learning dynamics in deep linear networks trained with bias terms. Even when initialized to zero, this simple architectural feature induces substantial changes in early dynamics.  We identify hallmarks of this early OCS phase and illustrate how these signatures are observed in deep linear networks and larger, more complex (and nonlinear) convolutional neural networks solving a hierarchical learning task based on MNIST and CIFAR10.  We explain these observations by proving that deep linear networks necessarily learn the OCS during early learning. To further probe the generality of our results, we train human learners over the course of three days on the category learning task. We then identify qualitative signatures of this early OCS phase in terms of the dynamics of true negative (correct-rejection) rates. Surprisingly, we find the same early reliance on the OCS in the behaviour of human learners. Finally, we show that learning of the OCS can emerge even in the absence of bias terms and is equivalently driven by generic correlations in the input data. Overall, our work suggests the OCS as a universal learning principle in supervised, error-corrective learning, and the mechanistic reasons for its prevalence.
\end{abstract}

\section{Introduction}

Neural networks trained with stochastic gradient descent (SGD) exhibit various \textit{simplicity biases}, where models tend to learn simple functions before more complex ones \citep{kalimeris_sgd_2019, rahaman_spectral_2019}. Simplicity biases hold significant theoretical interest as they provide an explanation for why deep networks generalize in practice, despite their expressivity and the presence of minima with arbitrarily high test error \citep{bhattamishra_simplicity_2023, valle-perez_deep_2019, zhang_understanding_2021}. 

The characterisation of simplicity biases is still incomplete. Some explanations appeal to distributional properties of input data, pointing out that SGD progressively learns increasingly higher-order moments \citep{refinetti_neural_2023, belrose_neural_2024}. Other approaches focus directly on the evolution of the network function, proposing that networks initially learn a classifier highly correlated with a linear model. Importantly, networks continue to perform well on examples correctly classified by this simple function, even when overfitting in later training \citep{kalimeris_sgd_2019}. This implies that dynamical simplicity biases help models generalise, by locking in initial knowledge that is not erased or forgotten during later training  \citep{braun_exact_2022, kalimeris_sgd_2019}. 

Deep linear networks have proven to be a valuable tool for studying simplicity biases. A key finding is that directions in the network function are learned in order of importance \citep{saxe_exact_2014,saxe2019mathematical}. This phenomenon, known as \textit{progressive differentiation}, connects modern deep learning theory to both to human child development and to the earliest connectionist models of semantic cognition \citep{rogers_semantic_2004, rumelhart_parallel_1986}.

Our contribution proposes a connection between these works by characterizing networks in the \textit{earliest} stages of learning in terms of input, output, and architecture. In the hierarchical setting by \citet{saxe2019mathematical}, we demonstrate both theoretically and empirically that neural networks initially learn via the output statistics of the data. This function has been termed the optimal constant solution (OCS) by \citet{kang_deep_2024}, who demonstrated that networks revert to the OCS when probed on out-of-sample inputs. Here, we demonstrate and prove how linear networks, when equipped with these bias terms, necessarily learn the OCS early in training. We furthermore highlight the practical relevance of these results by examining early learning dynamics in complex, non-linear architectures.

Biological learners also display behaviours that imply the input-independent learning of output statistics. In probability matching, responses mirror the probabilities of rewarded actions \citep{herrnstein_relative_1961, estes_probability_1964, estes_analysis_1954}. Learners often display non-stationary biases that are driven by the distribution of recent responses \citep{jones_role_2015, gold_relative_2008, verplanck_nonindependence_1952}. In paired-associates learning accuracy can depend not only on a learned input-output mapping but also on knowledge of the task structure \citep{hawker_influence_1964, bower_association_1962}. Humans also display simplicity biases and preferentially use simple over complex functions \citep{feldman_minimization_2000, goodman_rational_2008, chater_reconciling_1996,lombrozo_simplicity_2007, feldman_simplicity_2003}. However, relatively little attention has been devoted to the dynamics of these biases. We conduct experiments to determine whether humans replicate early reliance on the OCS.
\begin{figure}[t]
    \includegraphics[width=0.9\linewidth]{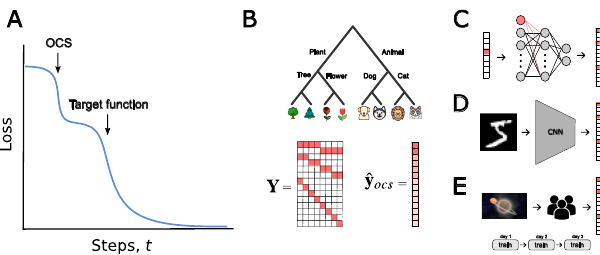}
    \centering
    \caption{Universal early learning of the optimal constant solution (OCS). \textbf{A} Graphical illustration of our hypothesis. \textbf{B} The hierarchical learning task used across learners and the OCS solution \(\hat{\mathbf{y}}_{ocs}\). \textbf{C} Learning task in linear networks with bias terms. \textbf{D} The task in CNNs. \textbf{E} The task for humans.}
    \label{fig:Task-and-setting}
\end{figure}
\subsection{Contributions}
\begin{itemize}

\item We devise exact solutions for learning dynamics to analyse linear networks with bias in the input layer. Even when initialized at zero, this component substantially alters \textit{early} learning dynamics.

\item We empirically characterise early learning in these linear networks as being dominated by average output statistics. We explain this result with a theoretical analysis which reveals that average output statistics are always learned first when the network contains bias terms.

\item We further highlight the practical relevance of these theoretical results in a hierarchical learning task for humans as well as complex, non-linear architectures by empirically demonstrating that all learners develop stereotypical response biases during early stages of training.

\item We conjecture on the basis of the developed theory that an early response bias should emerge even in absence of bias terms. We empirically demonstrate that learning of the OCS can indeed be driven purely by generic correlations in the input data.
\end{itemize}

\subsection{Related work}

\textbf{Deep linear networks.}
In deep linear networks analytical solutions have been obtained for certain initial conditions and datasets \citep{saxe_exact_2014, saxe2019mathematical, braun_exact_2022, fukumizu_effect_1998}. Progress has also been made in understanding linear network loss landscapes \citep{baldi_neural_1989} and generalisation ability \citep{lampinen_analytic_2019}. Despite their linearity these models display complex non-linear learning dynamics which reflect behaviours seen in non-linear models \citep{saxe2019mathematical}. Moreover, learning dynamics in such simple models have been argued to qualitatively resemble phenomena observed in the cognitive development of humans \citep{saxe2019mathematical, rogers_semantic_2004}. 

\textbf{Biological response biases.} Humans and animals routinely display response biases during perceptual learning and decision making tasks \citep{gold_relative_2008, jones_role_2015, liebana_garcia_striatal_2023, amitay_perceptual_2014, urai_choice_2019}.  
In these tasks decisions are frequently made in sequences where responses and feedback steer decisions beyond the provided perceptual evidence \citep{jones_role_2015, fan_neural_2024,gold_relative_2008,verplanck_nonindependence_1952,sugrue_matching_2004}. 
Non-stationary response biases can be driven by feedback on previous trials \citep{dutilh_how_2012,rabbitt_what_1977} or might reflect global beliefs about the statistics of a task \citep{fan_neural_2024, jones_role_2015}. Importantly, response biases are particularly pronounced in early learning \citep{jones_role_2015, gold_relative_2008, liebana_garcia_striatal_2023} and their influence appears to be strongest when uncertainty about the correct response is highest \citep{gold_relative_2008, fan_neural_2024}. 

\textbf{Simplicity biases in machine learning.}
Simplicity biases in neural networks have been studied extensively both theoretically \citep{Bordelon2020,Mei22Generalizationerrorrandom} and empirically \citep{bhattamishra_simplicity_2023, mingard_deep_2023}. Work on the \textit{distributional} simplicity bias emphasises the importance of input data and proposes that models learn via progressive exploitation of dataset moments \citep{refinetti_neural_2023, belrose_neural_2024}. On the other hand, neural networks have been found to express simpler functions during early training \citep{kalimeris_sgd_2019, refinetti_neural_2023, belrose_neural_2024, rahaman_spectral_2019}. Our work draws a connection between these findings and highlights how input statistics bias early learning towards output statistics. 

\subsection{Paper organisation} 
We initially review the linear network formalism in \cref{sec:lin-net-formalism} on which we base our theoretical analysis. In \cref{sec:early-dyn-bias} we derive learning dynamics for linear networks with bias terms trained on a classic hierarchical task and we document substantial changes in early dynamics. \cref{sec:early-dyn-ocs} characterizes this period of early learning empirically, and provides a theoretical explanation. We then in \cref{sec:early-dyn-ocs} validate the relevance of our findings for learning in complex models. \cref{sec:learners} demonstrates the prevalence of early OCS learning in humans. Finally, \cref{sec:input_corrs} further probes generality by considering natural datasets and models that do not strictly fulfil the previous theoretical assumptions. 

\section{Linear network preliminaries}
\label{sec:lin-net-formalism}

Here, we briefly review the analytical approach to learning dynamics in linear networks developed by \cite{saxe_exact_2014, saxe2019mathematical}. 
Consider a learning task in which a network is presented with input vectors \( \mathbf{x}_i \in \mathbb{R}^{N_{in}} \) that are associated to output vectors \( \mathbf{y}_i \in \mathbb{R}^{N_{out}} \). The total dataset consists of \( \{ \mathbf{x}_i, \mathbf{y}_i \}^N_{i=1} \) with \(N\) samples. For our setting we consider two layer linear networks where the forward pass computes \(\hat{\mathbf{y}}_i = \mathbf{W}^2\mathbf{W}^1\mathbf{x}_i\) and shallow networks with forward pass \(\hat{\mathbf{y}}_i =\mathbf{W}^s\mathbf{x}_i\). Here weight matrices are of dimension \( \mathbf{W}^1 \in \mathbb{R}^{N_{hid} \times N_{in}} \),  \( \mathbf{W}^2 \in \mathbb{R}^{N_{out} \times N_{hid}} \), and \( \mathbf{W}^s \in\mathbb{R}^{N_{out} \times N_{in}} \). We train our networks to minimise a squared error loss of the form \(\mathcal{L}(\hat{\mathbf{y}}) = \frac{1}{2} \sum_{i=1}^{N} \| \mathbf{y}_i - \hat{\mathbf{y}}_i \|^2.\)

We optimise networks using full batch-gradient descent in the gradient flow regime. When learning from small initial conditions, dynamics in these simple networks are solely dependent on the dataset input-output and input-input correlation matrices \citep{saxe_exact_2014}. Using singular value decomposition (SVD), these matrices can be expressed as 
\begin{equation}\label{eq:sim-diag}
\mathbf{\Sigma}^{yx}=\frac{1}{N} \mathbf{YX}^T = \mathbf{USV}^T, \quad \mathbf{\Sigma}^{x}=\frac{1}{N} \mathbf{XX}^T = \mathbf{VDV}^T. 
\end{equation}
Here \(\mathbf{X} \in \mathbb{R}^{N_{in} \times N}\) and \(\mathbf{Y} \in \mathbb{R}^{N_{out} \times N}\)contain the full set of input vectors and output vectors. Crucially, if  the right singular vectors \(\mathbf{V}^T\)of \(\mathbf{\Sigma}^{yx}\) diagonalise \(\mathbf{\Sigma}^{x}\) (see \cref{thm:commute}) then the full evolution of network weights for deep and shallow networks through time can be described as
\begin{equation}  
\label{eq:time-dep-svd}
\mathbf{W}^2(t)\mathbf{W}^1(t)=\mathbf{UA}(t)\mathbf{V}^T. 
\end{equation}
Here \(\mathbf{A}(t)\) is a diagonal matrix. The evolution of these diagonal values \(\mathbf{A}(t)_{\alpha \alpha} = a_{\alpha}(t)\) at each time-step \(t\) then follows a sigmoidal trajectory as expressed in \cref{eq:trajec_a}. For shallow networks we can similarly describe the evolution of the weight matrix \(\mathbf{W}^s(t)\) as \(\mathbf{UB}(t)\mathbf{V}^T\). Here the diagonal values \(\mathbf{B}(t)_{\alpha \alpha} = b_{\alpha}(t)\) evolve as seen in \cref{eq:trajec_b}

\noindent
\begin{minipage}{.5\textwidth}
\begin{equation}\label{eq:trajec_a}
    a_{\alpha}(t) = \frac{\nicefrac{s_{\alpha}}{d_{\alpha}}}{1 - (1 - \frac{s_{\alpha}}{d_{\alpha}a_{0}})e^{-\frac{2s_{\alpha}}{\tau}t}} 
\end{equation}
\end{minipage}%
\begin{minipage}{.5\textwidth}
\begin{equation}\label{eq:trajec_b}
    b_{\alpha}(t) = \frac{s_{\alpha}}{d_{\alpha}} (1 - e^{-\frac{d_{\alpha} }{\tau}t}) + b_{0}e^{-\frac{d_{\alpha} }{\tau}t}
\end{equation}
\end{minipage}

In \cref{eq:trajec_a} \(s_{\alpha} = \mathbf{S}_{\alpha\alpha}\) and \(d_{\alpha} = \mathbf{D}_{\alpha \alpha}\) denote the relevant singular values of \(\mathbf{\Sigma}^{yx}\) and the eigenvalues of \(\mathbf{\Sigma}^{x}\) respectively, \(a_{0}\) are the singular values at initialisation, and \(\tau=\frac{1}{N\epsilon}\) is the time constant where \(\epsilon\) is the learning rate. In \cref{eq:trajec_b} \(b_{0}\) is the initial condition given by the initialisation. Importantly, these relations reveal that singular values control learning speed. These solutions hinge on the diagonalisation of \(\mathbf{\Sigma}^{x}\) through \(\mathbf{V}\). Prior work has focused on the case of white inputs, i.e. \(\mathbf{\Sigma}^{x} = \mathbf{I}_N\) where  \(\mathbf{I}_N\) denotes the \(N \times N \) identity matrix. The solution holds trivially as any \(\mathbf{V}\) will orthogonalise \(\mathbf{\Sigma}^{x}\) \citep{saxe2019mathematical}. We discuss a relevant relaxation of this condition in \cref{thm:commute}. While solutions can be derived for some non-white inputs, little attention has been devoted to learning dynamics in these scenarios. We will show how these solutions apply when networks contain bias terms in the input layer.

\section{Exact learning dynamics with bias terms}
\label{sec:early-dyn-bias}
\begin{figure}
\includegraphics[width=.9\linewidth]{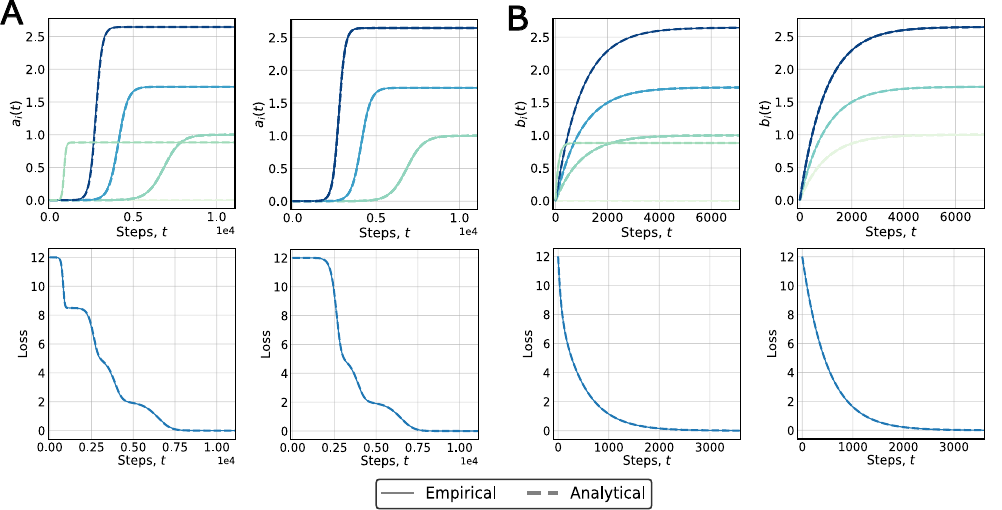}
    \centering
    \caption{Exact learning dynamics. \textbf{A} Deep linear networks with (left) and without (right) bias. \textbf{B} Shallow linear networks with (left) and without (right) bias. \textit{Top row:} Exact and simulated \(\mathbf{A}(t) \text{ and }\mathbf{B}(t)\) for deep and shallow networks respectively. \textit{Bottom row:} Exact and simulated loss.}
    \label{fig:Exact-Bias-solutions}
\end{figure}

In this section, we derive exact learning dynamics in linear networks with bias terms and analyse the resulting changes in the dynamics. This extension to the theory by \citet{saxe_exact_2014} forms the basis for our later discussion.  For simplicity, we focus on input bias terms and uncorrelated data, but explore bias terms in other layers and correlated inputs in \cref{app:learning-dynamics-bias-terms} and \cref{sec:input_corrs}, respectively.

\paragraph{Closed-form learning dynamics.}
We consider uncorrelated inputs \(\mathbf{X} = \mathbf{I}_N\) where \(\mathbf{I}_N\) denotes the \(N \times N\) identity matrix. Our linear network with a bias term in the input layer will compute \( \mathbf{W}^2(\mathbf{\tilde{W}}^1\mathbf{x}_i + \mathbf{\tilde{b}})\)  where \(\mathbf{\tilde{b}}\) are learnable bias terms.

The diagonalisation of $\mathbf{\Sigma}^{x}$
through $\mathbf{V}$ in \cref{eq:sim-diag} is not generally possible if the input layer includes bias terms. Here, we state the condition under which learning dynamics can be described in closed-form. 

\begin{restatable}[Feasibility of closed-form learning dynamics]{proposition}{commutemain}
\label{thm:commute}
For any input data $\mathbf{X}\in\mathbb{R}^{N_{in}\times N}$ and
output data $\mathbf{Y}\in\mathbb{R}^{N_{out}\times N}$ it is possible to diagonalize 
$\mathbf{\Sigma}^{x}$ by the right singular vectors $\mathbf{V}$
of $\mathbf{\Sigma}^{yx}$ if $\mathbf{Y}^{T}\mathbf{Y}$ and $\mathbf{X}^{T}\mathbf{X}$
commute. 
The converse holds true only if $\mathbf{X}$ has a left inverse. 
\end{restatable} 
A proof is given in \cref{app:commute}. We put this statement to use to assess the effect of a bias term on learning, building on the formalism from \cref{sec:lin-net-formalism}. To this end, we re-express the network weights as \(\mathbf{W}^1 =\begin{bmatrix} \mathbf{\tilde{b}} & \mathbf{\tilde{W}^1}
\end{bmatrix} \) with inputs defined as \(\mathbf{x}_i = \begin{bmatrix}
  1 &
  \mathbf{I}_i^T
\end{bmatrix}^T\) where \(\mathbf{I}_i\) denotes the \(i\)th column of the \(N \times N\) identity matrix (see \cref{app:Equivalence-of-bias}). Exact learning trajectories for the hierarchically structured output data are depicted in \cref{fig:Task-and-setting}B. 

Importantly, the introduction of  a bias term \(\bX \rightarrow \begin{bmatrix} \1 \: \bX\end{bmatrix}^T\) does not affect the commutativity of $\mathbf{X}^{T}\mathbf{X}$ and $\mathbf{Y}^{T}\mathbf{Y}$ for the hierarchical dataset, as the constant mode \(\1\) (i.e., a vector of 1s) is already an eigenvector to both these similarity matrices (see \cref{app:hierarchical_one}). In consequence, the analytical solutions in \cref{sec:lin-net-formalism} remain applicable. We generalise these considerations in \cref{sec:input_corrs} and \cref{app:one_in_XX_YY_general}. 

\cref{fig:Exact-Bias-solutions} shows that linear networks with bias terms have a distinctly different
early learning phase when compared to vanilla linear networks. While both models converge to a zero loss solution, we observe that the final network function with bias terms contains an additional non-zero singular singular value with their associated singular vectors.  We devote the next section to analyzing this change in the early dynamics.

\section{Bias terms drive early learning towards the optimal constant solution}
\label{sec:early-dyn-ocs}
\begin{figure}[t]
    \begin{center}
        \includegraphics[scale=0.7]{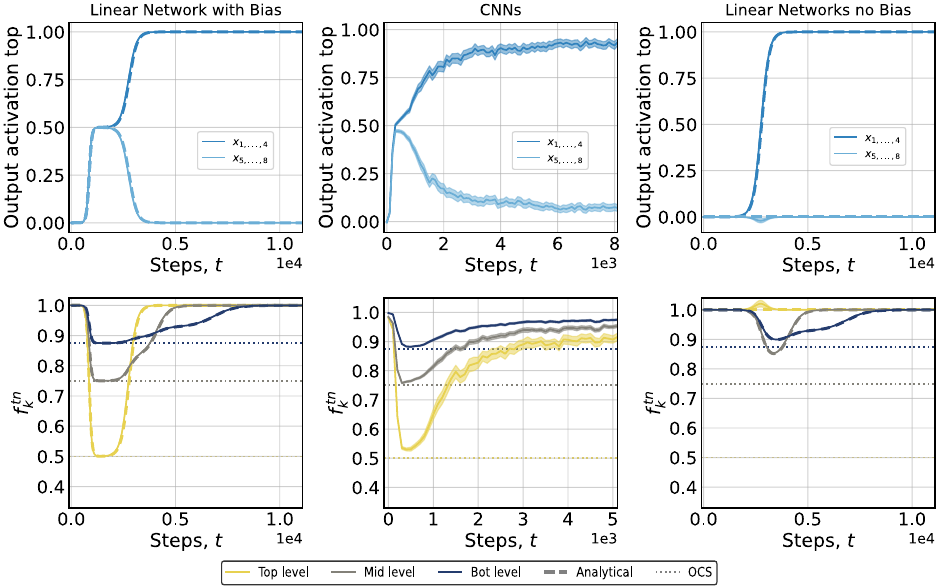}
    \end{center}
    \caption{Early learning is driven to the OCS. \textit{Top row:} Network outputs for a single output unit in response to all inputs \(\mathbf{x}_i\). \textit{Bottom row:} Continuous correct-rejection scores \(f^{tn}\) for the three hierarchical levels as indicated by colours (\cref{app:continous-TPR-and-TNR-Rates}) Performance approaches levels expected under the OCS (dotted lines). \textit{Left:} Linear networks with bias terms. \textit{Center:} CNNs. \textit{Right:} Linear networks without bias terms.}
    \label{fig:outputs}
\end{figure}

In this section, we qualitatively characterize what causes observed changes in early learning dynamics. We find that early learning dynamics are driven by average output statistics and provide a theoretical explanation. We then demonstrate the generality of this result by highlighting how early learning of average output statistics can be similarly observed in complex, non-linear architectures. 

A naive strategy to learning is to minimise error over a set of samples while disregarding information conveyed by the input. Previous work has recently termed this network function the optimal constant solution (OCS) \citep{kang_deep_2024}. The OCS can be formalised as \(\hat{\mathbf{y}}_{ocs} = \argmin_{\hat{\mathbf{y}} \in \mathbb{R}^{N_{out}}} \frac{1}{N} \sum_{i=1}^{N} \mathcal{L}(\hat{\mathbf{y}}, \mathbf{y}_i)\) and represents the optimal function \(\hat{\mathbf{y}}\) that is independent of input \(\mathbf{x}_i\). For mean-squared error,  it is straightforward to show that the minimiser is the average output  \(\hat{\mathbf{y}}_{ocs} = \frac{1}{N} \sum_i^N \mathbf{y}_i \eqqcolon \bar{\mathbf{y}}\).  

\textbf{Setup.} We train linear networks and Convolutional neural networks (CNN) on the hierarchical learning task in \cref{fig:Task-and-setting}. For CNNs we design a variation of the task whereby one-hot inputs are replaced with eight randomly sampled classes of MNIST \citep{li_deng_mnist_2012}. We use standard uniform Xavier initialization \citep{glorot_understanding_2010} and trained CNNs on an squared error loss. A full description of the CNN experiment and hyperparameter settings is deferred to \cref{app:CNN-experiments}. We there also replicate our results with CIFAR-10 \citep{krizhevsky_learning_nodate}, non-hierarchical tasks, alternative loss functions, and CelebA \citep{liu_deep_2015} in \cref{app:additional experiments}. We also show results for shallow networks in \cref{app:shallow-OCS}.

\subsection{Empirical evidence}
\label{sec:empirical-evidence}
We identify three separate empirical observations that support early learning of the OCS:

\paragraph{Indifference.} Linear networks and CNNs initially change outputs while not differentiating between input examples. In \cref{fig:outputs} (top) we show the empirical and analytical activation of an output unit associated with the highest level of the hierarchy for all \(\mathbf{x}_i\). Networks with and without bias terms learn to differentiate inputs correctly. However, networks with bias terms produce input-independent, non-zero outputs in early training as would be expected under the OCS. 

\paragraph{Performance.} We provide continuous correct-rejection scores for our task which approach levels expected under the OCS in \cref{fig:outputs} (bottom). Effectively, this metric measures wrong beliefs about the presence of target labels across the different levels of the hierarchy – it describes how strongly outputs \(\hat{\mathbf{y}}\) align with the desired outputs \(\mathbf{y}\) while focusing on zero entries only. We calculate the metric separately for the three hierarchical levels (details in \cref{app:continous-TPR-and-TNR-Rates}). Linear networks and CNNs with bias terms approach performance levels that would be provided by the OCS (dotted lines) for each of the three hierarchical levels. Linear network without bias terms do not produce this behaviour. 

\paragraph{OCS alignment.} The distance between outputs \(\hat{\mathbf{y}}_i\) of linear and non-linear networks and \(\mathbf{y}_{ocs}\) approaches zero in early training. \cref{fig:MNIST vs Orthogonal MNIST} (top)  shows how the \(L_1\) distance of sample-averaged network outputs and the OCS approaches zero before later converging to the desired network function.

\subsection{Theoretical explanation} \label{sec:theory_OCS}

In this section, we extend the linear network formalism to understand the mechanism behind early learning of the OCS. We first show how bias terms in the input layer are directly related to the OCS. Afterwards, we prove that the OCS is necessarily learned first in these settings.

\paragraph{The OCS is linked to shared properties.}
Having established the applicability of the linear network theory in \cref{sec:early-dyn-bias} we now seek to understand how the early bias towards the OCS emerges. To this end, notice how bias terms can be written in terms of the constant eigenmode \(\1\):

\begin{restatable}[The OCS is linked to shared properties]{proposition}{ocsonemain}
\label{thm:one_and_avg} If $\1$ is an eigenvector to the similarity matrix \(\mathbf{X}^{T}\mathbf{X}\in\mathbb{R}^{N\times N}\),
then the sample-average \(\bar{\mathbf{x}}=\frac{1}{N}\sum_{i=1}^{N}\mathbf{x}_{i}\)
will be an eigenvector to the correlation matrix \(\mathbf{X}\mathbf{X}^{T}\in\mathbb{R}^{N_{\text{in}}\times N_{\text{in}}}\)
with identical eigenvalue $\lambda$. An analogous statement applies
for  \(\mathbf{Y}^{T}\mathbf{Y}\) and \(\mathbf{Y}\mathbf{Y}^{T}\). The converse does not hold true in general. 
\end{restatable}

We prove this statement in \cref{app:ocsone}. Importantly, it establishes a connection between the feature and sample dimensions of $\mathbf{X}$ and $\mathbf{Y}$. If $\1$ is an eigenvector to $\bX^{T}\bX$ and $\bY^{T}\bY$ already, it implies that the addition of a bias term will directly add to its eigenvalue, \(s_{ocs}^{2}\rightarrow s_{ocs}^{2}+1\), even if it is initialized at zero. We show in \cref{app:hierarchical_one} that these assumptions on $\mathbf{X}$ and $\mathbf{Y}$ hold strictly for our hierarchical task design, and more generally relate to symmetry in the data (\cref{app:one_in_XX_YY_general}). We discuss in \cref{sec:input_corrs} how this property extends to natural datasets where strict symmetry is absent. 

Crucially, it now follows from \cref{thm:one_and_avg} that the time-dependent network correlation \(\hat{\mathbf{\Sigma}}^{yx}(t)=\mathbf{UA}(t)\mathbf{V}^{T}\) in \cref{eq:time-dep-svd} will contain a strongly amplified OCS mode $a_{ocs}(t)\mathbf{u}_{ocs}\mathbf{v}_{ocs}^{T}=a_{ocs}(t)\bar{\mathbf{y}}\bar{\mathbf{x}}^{T}$ by virtue of the modified singular value $\sqrt{s_{ocs}^{2}+1}$ entering \cref{eq:time-dep-svd} and thereby the network function.
Consequently, learning dynamics will be driven by the outer product of average input and output data. Moreover, this implies that given some input $\mathbf{x}_{i}$ to \cref{eq:time-dep-svd},
the network's OCS mode contributes 
\begin{equation}
\label{eq:prop_y_bar}
\mathbf{\hat y}_{ocs}(\mathbf{x}_{i}) = a_{ocs}(t)\mathbf{u}_{ocs}\mathbf{v}_{ocs}^{T}\mathbf{x}_{i} = a_{ocs}(t)\bar{\mathbf{y}}\bar{\mathbf{x}}^{T}\mathbf{x}_{i}\propto\bar{\mathbf{y}}.
\end{equation}
The OCS mode in the time-dependent network function will hence necessarily drive responses towards average
output statistics. Note that \cref{eq:prop_y_bar} also
highlights that the more an input example is aligned to average
inputs, the more the network's responses will reflect average outputs. In particular, this makes the expected output \(\mathbb{E}_{\mathbf{x}}[\hat{\mathbf{y}}(\mathbf{x})] \propto \bar{\mathbf{y}}\). Throughout learning, the evolution of $a_{ocs}(t)$
and scale-dependent alignment of $\mathbf{x}_{i}$ and $\bar{\mathbf{x}}$
will determine the network's reliance on the OCS mode. 

\paragraph{Early learning is biased by the OCS mode.} We established that network responses are driven by average
output statistics \(\bar{\mathbf{x}}\) and $\bar{\mathbf{y}}$, but why are \emph{early}
dynamics in particular influenced by the OCS? The learning speed of
the SVD modes in the time-dependent network function are controlled
by the magnitude of singular values $s_{\alpha}$ as seen in \cref{eq:trajec_a}. 
\begin{restatable}[Early learning is biased by the OCS mode]{theorem}{leadingmain}
    \label{thm:leading_ev}
    If \(\1\) is a joint non-degenerate eigenvector to positive input and output similarity matrices $\bX^T\bX$ and $\bY^T\bY$, the OCS mode $s_{ocs} \bar{\mathbf{y}} \bar{\mathbf{x}}^T$ will have \textit{leading} spectral weight $s_0 \equiv s_{ocs}$ in the SVD of the input-output correlation matrix $\mathbf{\Sigma}^{yx}$.
\end{restatable}
We prove this statement with help of the Perron-Frobenius theorem \citep{perron_zur_1907} in \cref{app:leading_ev}.    
Consequently,
the optimal constant mode is learned at a faster rate than remaining SVD components and transiently dominates the early network function. 
Notably, this applies to our task data $\bY^{T}\bY$ (see \cref{app:hierarchical_one}) and leads to characteristic learning signatures observed in
\cref{fig:outputs}.

\paragraph{Integrated formulation of architectural biases.}

\begin{wrapfigure}{r}{0.3\textwidth}
  \begin{center}
\includegraphics[width=0.3\textwidth]{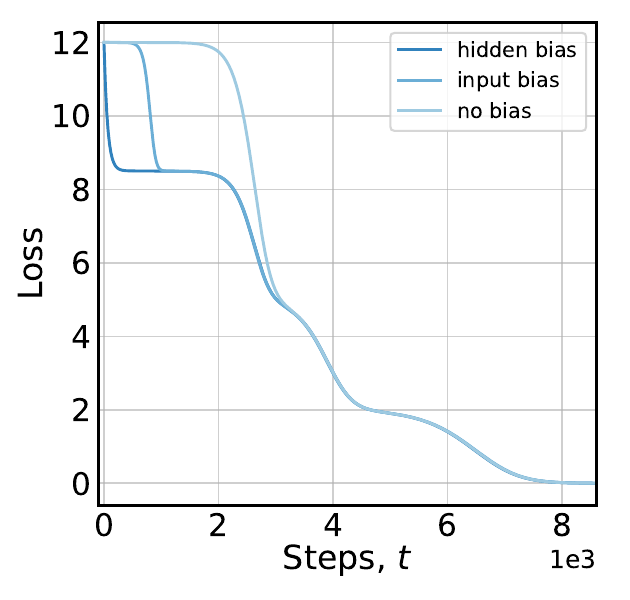}
  \end{center}
  \caption{Loss curves for different bias variations.}
  \label{fig:in-out-biases}
\end{wrapfigure}

\label{sec:NTK}
So far, we analysed how bias terms on the \textit{input} layer affect the singular value spectrum. 
Our empirical results in \cref{sec:empirical-evidence} suggest a more general dynamical bias towards the OCS stemming purely from architectural properties. We here use the neural tangent kernel 
$\NTK(\mathbf{x}_i,\mathbf{x}_{i'})=\sum_k\frac{d\hat{\mathbf{y}}_{i^{\phantom{\prime}}}^{\phantom{T}}}{d\theta_{k}}\frac{d\hat{\mathbf{y}}_{i^{\prime}}^{T}}{d\theta_{k}}$ 
\citep{jacot_neural_2018} to directly and comprehensively describe the affected time evolution of the network response $\frac{d}{dt}\hat{\mathbf{y}}_i = \NTK (\mathbf{y}_i-\hat{\mathbf{y}}_i)$ 
at the cost of a closed-form solution. Because changes in network outputs are proportional to the NTK it can been viewed as an architecture-induced learning rate \citep{roberts_principles_nodate}. 
For completeness, we now consider a network that contains input  $\mathbf{b}^1$ and output $\mathbf{b}^2$ bias terms. 
\begin{restatable}[NTK of linear networks with bias terms]{proposition}{NTKmain}
\label{thm:NTKmain}
\label{eq:NTK}
The neural tangent kernel of a two-layer linear network with input and output-layer bias $\hat{\mathbf{Y}}=\mathbf{W}^2 (\mathbf{W}^1\mathbf{X} + \mathbf{b}^1) + \mathbf{b}^2$ in the high-dimensional regime in early training reads
\vspace{-.5mm}
\begin{align}
\nonumber 
\mathsf{NTK}(\mathbf{X},\mathbf{X})=\sigma_{\mathbf{W}^2}^{2}\mathbf{I}_{N_{out}}\otimes\Bigl(2\mathbf{X}^{T}\mathbf{X}+\underbrace{\1\1^T}_{\leftrightarrow \mathbf{b}^{1}}\Bigr)+\mathbf{1}_{N_{out}}\mathbf{1}_{N_{out}}^{T}\otimes\underbrace{\1\1^T}_{\leftrightarrow \mathbf{b}^{2}}.
\end{align}
\end{restatable}
\vspace{-2mm}
The tensor product $\otimes$ separates the components that operate on output and sample space. 
We briefly review the $\NTK$ and derive this expression in \cref{app:NTK}. The highlighted terms originate from the bias term 
$\frac{d\hat{\mathbf{y}}_{i^{\phantom{\prime}}}^{\phantom{T}}}{d\mathbf{b}}\frac{d\hat{\mathbf{y}}_{i^{\prime}}^{T}}{d\mathbf{b}}$ 
entering the $\NTK$, manifesting in the appearance of the constant mode $\1$. Importantly, these terms do not scale with the size of the learned bias – they are present even if the bias is initialized at zero. Intuitively, their contribution stems from the architecture's \textit{potential} to learn a bias, enabling rapid changes in output $\hat{\mathbf{y}}$.
The $\NTK$ also reveals a qualitative difference between input and output bias: Whereas the term that is induced by $\mathbf{b}^1$ shows attenuated growth due to the multiplication by the weights of initial scale $\sigma_{\mathbf{W}^2}\ll 1$, the output bias $\mathbf{b}^2$ immediately changes the output significantly. Loss curves which demonstrate the effect of different bias terms are displayed in \cref{fig:in-out-biases}.

To recapitulate this section: We first rephrased a learnable bias term in the architecture as a shared feature in the input data. We then found that the associated singular value in \cref{eq:time-dep-svd} drives the learned network function towards the OCS (\cref{eq:prop_y_bar}). Finally, we proved that the bias affects \textit{early} learning. We furthermore show that this tendency is even more pronounced for bias terms in later layers of the network. Overall, these results demonstrate how architectural bias terms induce early OCS learning. 

\section{Correspondence of linear networks, complex models, and humans}
\label{sec:learners}

\begin{figure}
    \includegraphics[width=1\linewidth]{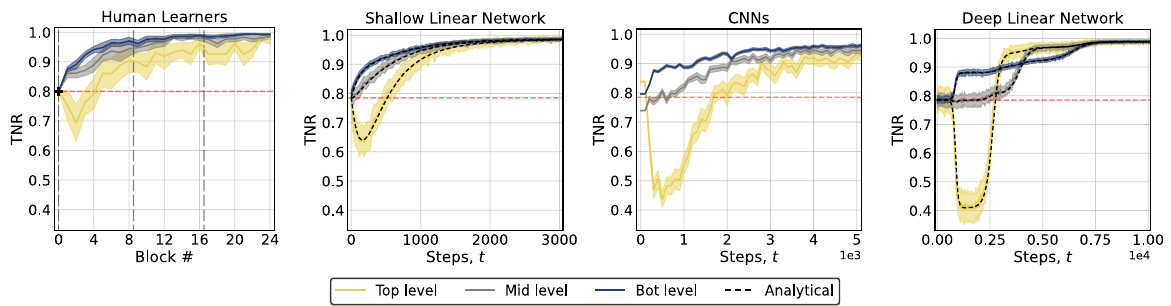}
    \centering
    \caption{Early response bias towards the OCS across learners in the hierarchical learning task. Dashed red line indicates chance performance. Dashed vertical grey lines indicates breaks between days for human learners. Colour code as in \cref{fig:outputs}.}
    \label{fig:Human-net-correspondence}
\end{figure}

In this section, we demonstrate how human learners, linear networks, and non-linear architectures show strong correspondences in their early learning on the hierarchical task displayed in \cref{fig:Task-and-setting}.

\textbf{Setup.} The hierarchical learning task has previously been used extensively in the study of semantic cognition \citep{rogers_semantic_2004} and requires learners to develop a hierarchical one-to-many mapping as seen in \cref{fig:Task-and-setting}B. We adapted the task for human learners while maintaining the underlying structure: Input stimuli were represented as different classes of planets and output labels were represented as a set of plant images (see \cref{fig:Task-and-setting}E and \cref{fig:human_task_supplement}). We also trained CNNs as in \cref{sec:early-dyn-ocs}. Importantly, the hierarchical structure results in a non-uniform distribution of labels with average labels equal to \(\mathbf{y}_{ocs}\). 
Human learners received supervised training over three days. A full description of the experimental paradigm is given \cref{app:human-experiment}.  We then compute correct-rejection scores (TNR) as in \cref{sec:early-dyn-ocs} while splitting performance across the hierarchical levels (details of metric in \cref{app:continous-TPR-and-TNR-Rates}). 

Neural networks produced continuous outputs in \(\mathbb{R}^{N_{out}}\)  while humans responded via discrete button clicks in \(\{0,1\}^{N_{out}}\) . To compare networks to humans we discretize network outputs by stochastically mapping continuous responses into \(\{0,1\}^{N_{out}}\). Procedure details are given in \cref{app:discretization}. 

\textbf{Results.} The key results of our experiments are presented in \cref{fig:Human-net-correspondence}. Intriguingly, we find that human learners, linear networks, and CNNs all display characteristic early response biases. Note that chance TNR is equal between all three levels of the hierarchy. Biological as well as artificial learners display an initial "drop" in TNR at the top level of the output hierarchy.  The result indicates a general lack of specificity and an overly liberal response criterion for output labels on the top level of the hierarchy. 
To appreciate the significance of this result it is important to understand that the task can be learned without the development of these early response biases: In particular, linear networks without bias terms do not show this behaviour (see \cref{app:TNR-TPR-no-bias}). Surprisingly, the human response signature demonstrates that these learners, just as artifical networks, display an early bias towards the OCS. We conjecture that early learning of the OCS might be a universal phenomenon that emerges during error-corrective training. We replicate the human result with a second cohort of learners in  \cref{app:human-experiment}. 
Notable is also the difference between shallow and deep linear networks. Response biases seem more transient in shallow networks and appear to more closely mirror human learners. However, quantitative comparisons are challenging due to inherently differing learning timescales.

\section{Generic input correlations can equivalently drive OCS learning}\label{sec:input_corrs}

We have established how the earliest phase of learning in linear networks is driven by the OCS. Crucially, in linear networks OCS learning hinges on bias terms in the network architecture. However, in non-linear architectures, such as CNNs, the network is driven towards the OCS even in the absence of bias terms (\cref{fig:MNIST vs Orthogonal MNIST}B, Top).  The appearance of the data term $\bX^T\bX$ in \cref{thm:NTKmain} suggests an equivalent effect that is induced by the data itself.
\begin{restatable}[Input correlations induce early OCS]{corollary}{inputcorrmain}
\label{thm:inputcorrmain}
If $\1$ is an eigenvector of the data similarity matrix $\bX^T\bX$ with non-degenerate eigenvalue $s_0$, then the OCS response during early learning will be driven according to its magnitude. 
\end{restatable}
This statement follows directly from the joint diagonalisation of \cref{eq:sim-diag} and subsequent projection onto the OCS $\mathbf{y}_{ocs}$. We furthermore hypothesize that neural networks will be driven towards the OCS if training data contains more generic input correlations where $\1$ is not an exact eigenvector.

\textbf{Setup.} We trained CNNs on the hierarchical task in  \cref{sec:early-dyn-ocs}. Inputs were given by eight randomly sampled classes of MNIST (\cref{fig:MNIST vs Orthogonal MNIST}B, Top). To isolate the effect of input correlations we created a second dataset where randomly sampled classes of MNIST were copied on orthogonal subspaces of a larger image (\cref{fig:MNIST vs Orthogonal MNIST}B, bottom). Importantly, this procedure removes all between-class correlations.

\begin{figure}
\includegraphics[width=0.85\linewidth]{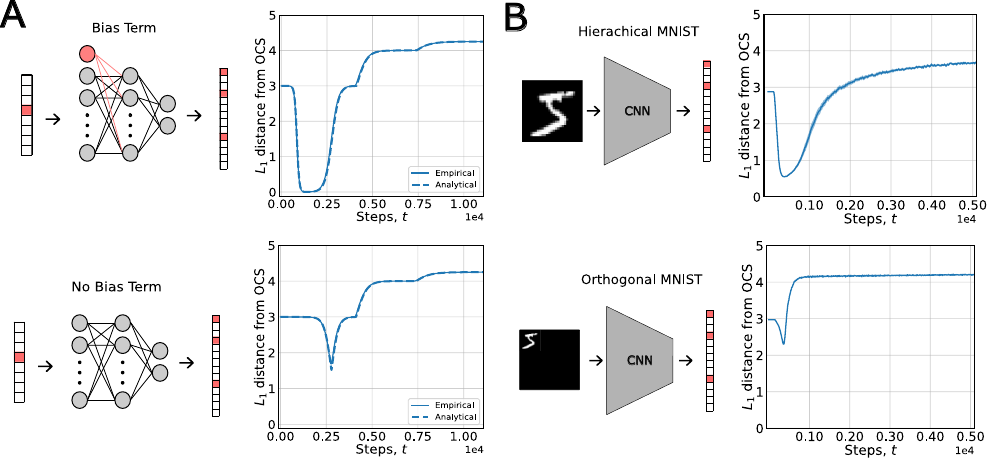}
    \centering
    \caption{Distance from the OCS. \textbf{A} Linear networks with (top) and without (bottom) bias terms. \textbf{B} CNNs without bias terms. \textit{Top} normal inputs. \textit{Bottom} "orthogonal" inputs.}
    \label{fig:MNIST vs Orthogonal MNIST}
\end{figure}

\textbf{Results.} The main result of our experiment is displayed in \cref{fig:MNIST vs Orthogonal MNIST}. CNNs which learn from standard MNIST images are strongly driven towards the OCS. In contrast, early dynamics for the "orthogonal" MNIST do not display this tendency. Strikingly, the early dynamics with standard MNIST classes are highly similar to those observed in linear network with bias terms, while the dynamics for the latter task resemble those seen in the linear network without this feature. 
To verify that generic input correlations are indeed causing these differences we explore the eigenspectrum of the data correlation matrices. We sample 100 images from all 10 classes and compute a correlation matrix \(\mathbf{X}^T\mathbf{X}\) from flattened images. First, we find that the eigenspectrum for standard MNIST images is dominated by a single eigenvector (\cref{fig:eigenspectra}, top-left). In contrast, the eigenspectrum of the orthogonal MNIST task does not display this property (\cref{fig:eigenspectra}, top-center). Further, recall that input bias terms lead to a non-degenerate constant eigenvector \(\mathbf{1}_N\) in the input correlation matrix (\cref{sec:early-dyn-ocs}). Similarly, we find that the first eigenvector \(\mathbf{v}_1\) of \(\mathbf{X}^T\mathbf{X}\) is indeed highly aligned to $\1$ (\cref{fig:eigenspectra}, right), whereas this is not the case in the orthogonal MNIST. We additionally show similar results for CIFAR-10 and CelebA. Theoretical considerations suggest that these correlations originate from an approximate symmetry in the data (see \cref{app:one_in_XX_YY_general}) and might hence be ubiquitous beyond these datasets. 

Overall, we here demonstrated that early learning of the OCS can be driven by properties of the architecture (bias terms) or data (input correlations). Our results also highlight that input correlations are a common feature of standard datasets: Early learning of the OCS might be inevitable in many settings. To see a practical implication of these results we briefly discuss fairness implications of OCS learning in \cref{app:linear-net-classimbalance}.

\begin{figure}[h]
    \includegraphics[width=0.85\linewidth]{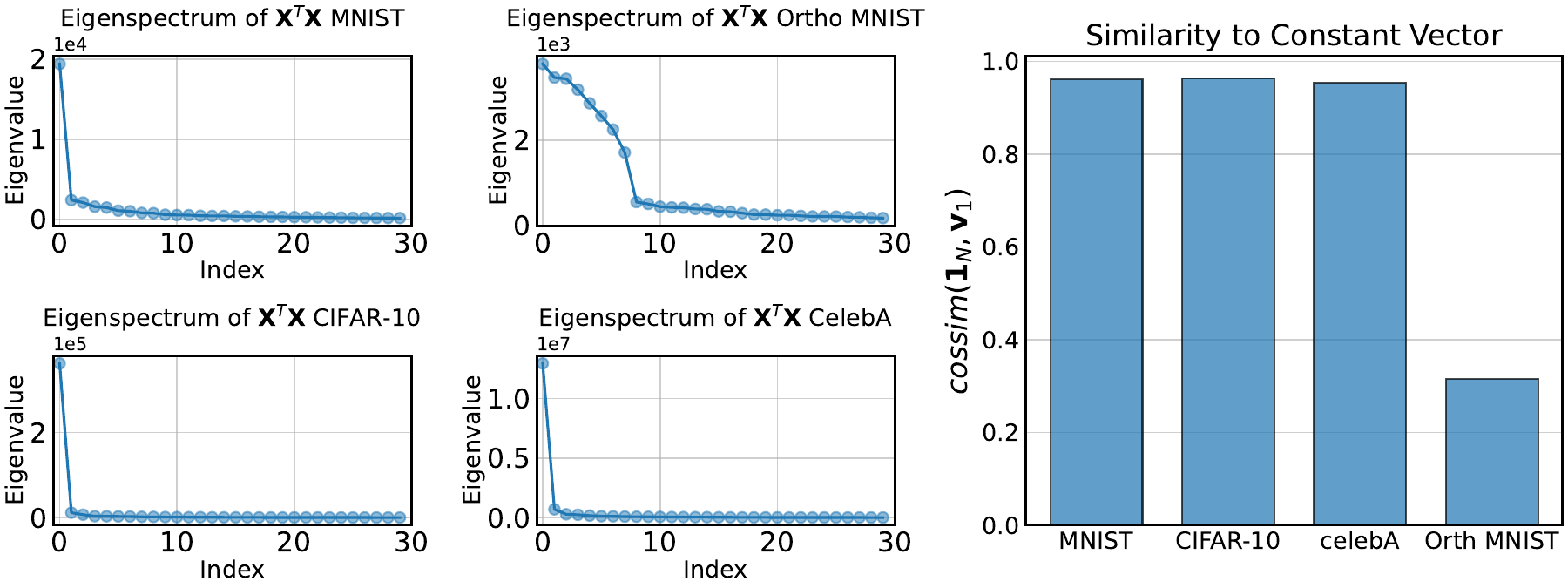}
    \centering
    \caption{Dataset eigenspectra and constancy of first eigenvector. \textit{Left:} Eigenspectra of \(\mathbf{X}^T\mathbf{X}\)  for different datasets. \textit{Right:} Alignment of first eigenvector in \(\bX^T\bX\) with the constant vector \(\mathbf{1}_N\).}
    \label{fig:eigenspectra}
\end{figure}

\section{Discussion}
We found that the inclusion of bias terms in linear networks shifts early learning towards the OCS, even when initialized at zero. We demonstrated that early, input-independent simplicity biases occur in practice, affecting both non-linear networks and human learners. Our contribution complements prior work on simplicity biases by highlighting factors that drive networks in the \textit{earliest} stages of learning; connecting input, output, and architecture. Our findings highlight how simple linear networks can serve as useful tools to investigate simplicity biases in significantly more complex systems.
A limitation of our work is the focus on qualitative comparisons between linear networks and non-linear systems. We chose linear networks to allow for a rigorous description of the dynamics of learning. Methods from mean-field theory may provide a more precise tool to analyze a wider range of systems directly. Secondly, the ambiguity between architecture and data in driving the OCS does not allow us to determine the underlying mechanism in human learners. Future studies might address this limitation by manipulating correlations in stimuli or by recording of neural data. Our proposed framework furthermore lends itself to theoretically studying the effect of noise on the network, for example by random matrix theory. We leave this investigation for future work.

\begin{ack}
We thank Satwik Bhattamishra, Aaditya Singh, Clémentine Dominé, Lukas Braun, Devon Jarvis, and Kevin Huang for useful feedback, discussions, and comments on earlier versions of this work.
This work was funded by a European Research Council (ERC) Consolidator Award (725937) to C.S., a Wellcome Trust Discovery Award (227928/Z/23/Z) to C.S., and a UKRI ESRC Grand Union Doctoral training partnership stipend awarded to J.R. This work was also supported by a Schmidt Science Polymath Award to A.S., and the Sainsbury Wellcome Centre Core Grant from Wellcome (219627/Z/19/Z) and the Gatsby Charitable Foundation (GAT3850). A.S. is a CIFAR Azrieli Global Scholar in the Learning in Machines \& Brains program.
\end{ack}

\section*{Author Contributions}
The conceptualisation of the project was developed by J.R. with C.S. and A.S. providing supervision. J.R. is responsible for all empirical results in networks, the human experiment, and initial theoretical ideas. J.B. and J.R. collaboratively devised the main theoretical results in the paper. 
J.B. primarily developed the theoretical presentation in the appendix. 
The initial draft was written by J.R. with further iterations and appendices written collaboratively by J.B. and J.R. All authors contributed to polishing of the draft and C.S. and A.S. provided supervision on all aspects of the project.
\newpage
\bibliography{References,bib_jan}

\begin{thebibliography}{56}
\providecommand{\natexlab}[1]{#1}
\providecommand{\url}[1]{\texttt{#1}}
\expandafter\ifx\csname urlstyle\endcsname\relax
  \providecommand{\doi}[1]{doi: #1}\else
  \providecommand{\doi}{doi: \begingroup \urlstyle{rm}\Url}\fi

\bibitem[Kalimeris et~al.(2019)Kalimeris, Kaplun, Nakkiran, Edelman, Yang, Barak, and Zhang]{kalimeris_sgd_2019}
Dimitris Kalimeris, Gal Kaplun, Preetum Nakkiran, Benjamin Edelman, Tristan Yang, Boaz Barak, and Haofeng Zhang.
\newblock {SGD} on {Neural} {Networks} {Learns} {Functions} of {Increasing} {Complexity}.
\newblock In \emph{Advances in {Neural} {Information} {Processing} {Systems}}, volume~32. Curran Associates, Inc., 2019.

\bibitem[Rahaman et~al.(2019)Rahaman, Baratin, Arpit, Draxler, Lin, Hamprecht, Bengio, and Courville]{rahaman_spectral_2019}
Nasim Rahaman, Aristide Baratin, Devansh Arpit, Felix Draxler, Min Lin, Fred Hamprecht, Yoshua Bengio, and Aaron Courville.
\newblock On the {Spectral} {Bias} of {Neural} {Networks}.
\newblock In \emph{Proceedings of the 36th {International} {Conference} on {Machine} {Learning}}, pages 5301--5310. PMLR, May 2019.
\newblock URL \url{https://proceedings.mlr.press/v97/rahaman19a.html}.
\newblock ISSN: 2640-3498.

\bibitem[Bhattamishra et~al.(2023)Bhattamishra, Patel, Kanade, and Blunsom]{bhattamishra_simplicity_2023}
Satwik Bhattamishra, Arkil Patel, Varun Kanade, and Phil Blunsom.
\newblock Simplicity {Bias} in {Transformers} and their {Ability} to {Learn} {Sparse} {Boolean} {Functions}, July 2023.
\newblock URL \url{http://arxiv.org/abs/2211.12316}.
\newblock arXiv:2211.12316 [cs].

\bibitem[Valle-Pérez et~al.(2019)Valle-Pérez, Camargo, and Louis]{valle-perez_deep_2019}
Guillermo Valle-Pérez, Chico~Q. Camargo, and Ard~A. Louis.
\newblock Deep learning generalizes because the parameter-function map is biased towards simple functions, April 2019.
\newblock URL \url{http://arxiv.org/abs/1805.08522}.
\newblock arXiv:1805.08522 [cs, stat].

\bibitem[Zhang et~al.(2021)Zhang, Bengio, Hardt, Recht, and Vinyals]{zhang_understanding_2021}
Chiyuan Zhang, Samy Bengio, Moritz Hardt, Benjamin Recht, and Oriol Vinyals.
\newblock Understanding deep learning (still) requires rethinking generalization.
\newblock \emph{Communications of the ACM}, 64\penalty0 (3):\penalty0 107--115, March 2021.
\newblock ISSN 0001-0782, 1557-7317.
\newblock \doi{10.1145/3446776}.
\newblock URL \url{https://dl.acm.org/doi/10.1145/3446776}.

\bibitem[Refinetti et~al.(2023)Refinetti, Ingrosso, and Goldt]{refinetti_neural_2023}
Maria Refinetti, Alessandro Ingrosso, and Sebastian Goldt.
\newblock Neural networks trained with {SGD} learn distributions of increasing complexity.
\newblock In \emph{Proceedings of the 40th {International} {Conference} on {Machine} {Learning}}, pages 28843--28863. PMLR, July 2023.
\newblock URL \url{https://proceedings.mlr.press/v202/refinetti23a.html}.
\newblock ISSN: 2640-3498.

\bibitem[Belrose et~al.(2024)Belrose, Pope, Quirke, Mallen, and Fern]{belrose_neural_2024}
Nora Belrose, Quintin Pope, Lucia Quirke, Alex Mallen, and Xiaoli Fern.
\newblock Neural {Networks} {Learn} {Statistics} of {Increasing} {Complexity}, February 2024.
\newblock URL \url{http://arxiv.org/abs/2402.04362}.
\newblock arXiv:2402.04362 [cs].

\bibitem[Braun et~al.(2022)Braun, Dominé, Fitzgerald, and Saxe]{braun_exact_2022}
Lukas Braun, Clémentine Dominé, James Fitzgerald, and Andrew Saxe.
\newblock Exact learning dynamics of deep linear networks with prior knowledge.
\newblock \emph{Advances in Neural Information Processing Systems}, 35:\penalty0 6615--6629, December 2022.
\newblock URL \url{https://proceedings.neurips.cc/paper_files/paper/2022/hash/2b3bb2c95195130977a51b3bb251c40a-Abstract-Conference.html}.

\bibitem[Saxe et~al.(2014)Saxe, McClelland, and Ganguli]{saxe_exact_2014}
Andrew~M. Saxe, James~L. McClelland, and Surya Ganguli.
\newblock Exact solutions to the nonlinear dynamics of learning in deep linear neural networks, February 2014.
\newblock URL \url{http://arxiv.org/abs/1312.6120}.
\newblock arXiv:1312.6120 [cond-mat, q-bio, stat].

\bibitem[Saxe et~al.(2019)Saxe, McClelland, and Ganguli]{saxe2019mathematical}
Andrew~M. Saxe, James~L. McClelland, and Surya Ganguli.
\newblock A mathematical theory of semantic development in deep neural networks.
\newblock \emph{Proceedings of the National Academy of Sciences}, 116\penalty0 (23):\penalty0 11537--11546, June 2019.
\newblock \doi{10.1073/pnas.1820226116}.
\newblock URL \url{https://www.pnas.org/doi/full/10.1073/pnas.1820226116}.
\newblock Publisher: Proceedings of the National Academy of Sciences.

\bibitem[Rogers and McClelland(2004)]{rogers_semantic_2004}
Timothy~T. Rogers and James~L. McClelland.
\newblock \emph{Semantic {Cognition}: {A} {Parallel} {Distributed} {Processing} {Approach}}.
\newblock MIT Press, 2004.
\newblock ISBN 978-0-262-18239-3.
\newblock Google-Books-ID: AmB33Uz2MVAC.

\bibitem[Rumelhart et~al.(1986)Rumelhart, McClelland, and Group]{rumelhart_parallel_1986}
David~E. Rumelhart, James~L. McClelland, and PDP~Research Group.
\newblock \emph{Parallel distributed processing, volume 1: {Explorations} in the microstructure of cognition: {Foundations}}.
\newblock The MIT press, 1986.
\newblock URL \url{https://scholar.google.com/scholar?cluster=13839636846206420541&hl=en&oi=scholarr}.

\bibitem[Kang et~al.(2024)Kang, Setlur, Tomlin, and Levine]{kang_deep_2024}
Katie Kang, Amrith Setlur, Claire Tomlin, and Sergey Levine.
\newblock Deep {Neural} {Networks} {Tend} {To} {Extrapolate} {Predictably}, March 2024.
\newblock URL \url{http://arxiv.org/abs/2310.00873}.
\newblock arXiv:2310.00873 [cs].

\bibitem[Herrnstein(1961)]{herrnstein_relative_1961}
R.~J. Herrnstein.
\newblock Relative and absolute strength of response as a function of frequency of reinforcement,.
\newblock \emph{Journal of the Experimental Analysis of Behavior}, 4\penalty0 (3):\penalty0 267--272, July 1961.
\newblock ISSN 0022-5002.
\newblock \doi{10.1901/jeab.1961.4-267}.
\newblock URL \url{https://www.ncbi.nlm.nih.gov/pmc/articles/PMC1404074/}.

\bibitem[Estes(1964)]{estes_probability_1964}
William~K. Estes.
\newblock Probability {Learning} and {Sequence} learning.
\newblock In Arthur~W. Melton, editor, \emph{Categories of {Human} {Learning}}, pages 89--128. Academic Press, January 1964.
\newblock ISBN 978-1-4832-3145-7.
\newblock \doi{10.1016/B978-1-4832-3145-7.50010-8}.
\newblock URL \url{https://www.sciencedirect.com/science/article/pii/B9781483231457500108}.

\bibitem[Estes and Straughan(1954)]{estes_analysis_1954}
W.~K. Estes and J.~H. Straughan.
\newblock Analysis of a verbal conditioning situation in terms of statistical learning theory.
\newblock \emph{Journal of Experimental Psychology}, 47\penalty0 (4):\penalty0 225--234, 1954.
\newblock ISSN 0022-1015.
\newblock \doi{10.1037/h0060989}.
\newblock URL \url{https://doi.apa.org/doi/10.1037/h0060989}.

\bibitem[Jones et~al.(2015)Jones, Moore, Shub, and Amitay]{jones_role_2015}
Pete~R. Jones, David~R. Moore, Daniel~E. Shub, and Sygal Amitay.
\newblock The role of response bias in perceptual learning.
\newblock \emph{Journal of Experimental Psychology: Learning, Memory, and Cognition}, 41\penalty0 (5):\penalty0 1456--1470, September 2015.
\newblock ISSN 1939-1285, 0278-7393.
\newblock \doi{10.1037/xlm0000111}.
\newblock URL \url{https://doi.apa.org/doi/10.1037/xlm0000111}.

\bibitem[Gold et~al.(2008)Gold, Law, Connolly, and Bennur]{gold_relative_2008}
Joshua~I. Gold, Chi-Tat Law, Patrick Connolly, and Sharath Bennur.
\newblock The {Relative} {Influences} of {Priors} and {Sensory} {Evidence} on an {Oculomotor} {Decision} {Variable} {During} {Perceptual} {Learning}.
\newblock \emph{Journal of Neurophysiology}, 100\penalty0 (5):\penalty0 2653--2668, November 2008.
\newblock ISSN 0022-3077, 1522-1598.
\newblock \doi{10.1152/jn.90629.2008}.
\newblock URL \url{https://www.physiology.org/doi/10.1152/jn.90629.2008}.

\bibitem[Verplanck et~al.(1952)Verplanck, Collier, and Cotton]{verplanck_nonindependence_1952}
William~S. Verplanck, George~H. Collier, and John~W. Cotton.
\newblock Nonindependence of successive responses in measurements of the visual threshold.
\newblock \emph{Journal of Experimental Psychology}, 44\penalty0 (4):\penalty0 273--282, 1952.
\newblock ISSN 0022-1015.
\newblock \doi{10.1037/h0054948}.
\newblock URL \url{https://doi.apa.org/doi/10.1037/h0054948}.

\bibitem[Hawker(1964)]{hawker_influence_1964}
James~R. Hawker.
\newblock The influence of training procedure and other task variables in paired-associate learning.
\newblock \emph{Journal of Verbal Learning and Verbal Behavior}, 3\penalty0 (1):\penalty0 70--76, February 1964.
\newblock ISSN 0022-5371.
\newblock \doi{10.1016/S0022-5371(64)80060-8}.
\newblock URL \url{https://www.sciencedirect.com/science/article/pii/S0022537164800608}.

\bibitem[Bower(1962)]{bower_association_1962}
Gordon~H. Bower.
\newblock An association model for response and training variables in paired-associate learning.
\newblock \emph{Psychological Review}, 69\penalty0 (1):\penalty0 34--53, January 1962.
\newblock ISSN 1939-1471, 0033-295X.
\newblock \doi{10.1037/h0039023}.
\newblock URL \url{https://doi.apa.org/doi/10.1037/h0039023}.

\bibitem[Feldman(2000)]{feldman_minimization_2000}
Jacob Feldman.
\newblock Minimization of {Boolean} complexity in human concept learning.
\newblock \emph{Nature}, 407\penalty0 (6804):\penalty0 630--633, October 2000.
\newblock ISSN 0028-0836, 1476-4687.
\newblock \doi{10.1038/35036586}.
\newblock URL \url{https://www.nature.com/articles/35036586}.

\bibitem[Goodman et~al.(2008)Goodman, Tenenbaum, Feldman, and Griffiths]{goodman_rational_2008}
Noah~D. Goodman, Joshua~B. Tenenbaum, Jacob Feldman, and Thomas~L. Griffiths.
\newblock A {Rational} {Analysis} of {Rule}-{Based} {Concept} {Learning}.
\newblock \emph{Cognitive Science}, 32\penalty0 (1):\penalty0 108--154, 2008.
\newblock ISSN 1551-6709.
\newblock \doi{10.1080/03640210701802071}.
\newblock URL \url{https://onlinelibrary.wiley.com/doi/abs/10.1080/03640210701802071}.
\newblock \_eprint: https://onlinelibrary.wiley.com/doi/pdf/10.1080/03640210701802071.

\bibitem[Chater(1996)]{chater_reconciling_1996}
Nick Chater.
\newblock Reconciling {Simplicity} and {Likelihood} {Principles} in {Perceptual} {Organization}.
\newblock \emph{Psychological review}, 103:\penalty0 566--81, July 1996.
\newblock \doi{10.1037/0033-295X.103.3.566}.

\bibitem[Lombrozo(2007)]{lombrozo_simplicity_2007}
Tania Lombrozo.
\newblock Simplicity and probability in causal explanation.
\newblock \emph{Cognitive Psychology}, 55\penalty0 (3):\penalty0 232--257, November 2007.
\newblock ISSN 0010-0285.
\newblock \doi{10.1016/j.cogpsych.2006.09.006}.
\newblock URL \url{https://www.sciencedirect.com/science/article/pii/S0010028506000739}.

\bibitem[Feldman(2003)]{feldman_simplicity_2003}
Jacob Feldman.
\newblock The {Simplicity} {Principle} in {Human} {Concept} {Learning}.
\newblock \emph{Current Directions in Psychological Science}, 12\penalty0 (6):\penalty0 227--232, December 2003.
\newblock ISSN 0963-7214.
\newblock \doi{10.1046/j.0963-7214.2003.01267.x}.
\newblock URL \url{https://doi.org/10.1046/j.0963-7214.2003.01267.x}.
\newblock Publisher: SAGE Publications Inc.

\bibitem[Fukumizu(1998)]{fukumizu_effect_1998}
Kenji Fukumizu.
\newblock Effect {Of} {Batch} {Learning} {In} {Multilayer} {Neural} {Networks}.
\newblock June 1998.

\bibitem[Baldi and Hornik(1989)]{baldi_neural_1989}
Pierre Baldi and Kurt Hornik.
\newblock Neural networks and principal component analysis: {Learning} from examples without local minima.
\newblock \emph{Neural Networks}, 2\penalty0 (1):\penalty0 53--58, January 1989.
\newblock ISSN 0893-6080.
\newblock \doi{10.1016/0893-6080(89)90014-2}.
\newblock URL \url{https://www.sciencedirect.com/science/article/pii/0893608089900142}.

\bibitem[Lampinen and Ganguli(2019)]{lampinen_analytic_2019}
Andrew~K. Lampinen and Surya Ganguli.
\newblock An analytic theory of generalization dynamics and transfer learning in deep linear networks, January 2019.
\newblock URL \url{http://arxiv.org/abs/1809.10374}.
\newblock arXiv:1809.10374 [cs, stat].

\bibitem[Liebana~Garcia et~al.(2023)Liebana~Garcia, Laffere, Toschi, Schilling, Podlaski, Fritsche, Zatka-Haas, Li, Bogacz, Saxe, and Lak]{liebana_garcia_striatal_2023}
Samuel Liebana~Garcia, Aeron Laffere, Chiara Toschi, Louisa Schilling, Jacek Podlaski, Matthias Fritsche, Peter Zatka-Haas, Yulong Li, Rafal Bogacz, Andrew Saxe, and Armin Lak.
\newblock Striatal dopamine reflects individual long-term learning trajectories, December 2023.
\newblock URL \url{http://biorxiv.org/lookup/doi/10.1101/2023.12.14.571653}.

\bibitem[Amitay et~al.(2014)Amitay, Zhang, Jones, and Moore]{amitay_perceptual_2014}
Sygal Amitay, Yu-Xuan Zhang, Pete~R. Jones, and David~R. Moore.
\newblock Perceptual learning: {Top} to bottom.
\newblock \emph{Vision Research}, 99:\penalty0 69--77, June 2014.
\newblock ISSN 0042-6989.
\newblock \doi{10.1016/j.visres.2013.11.006}.
\newblock URL \url{https://www.sciencedirect.com/science/article/pii/S0042698913002800}.

\bibitem[Urai et~al.(2019)Urai, de~Gee, Tsetsos, and Donner]{urai_choice_2019}
Anne~E Urai, Jan~Willem de~Gee, Konstantinos Tsetsos, and Tobias~H Donner.
\newblock Choice history biases subsequent evidence accumulation.
\newblock \emph{eLife}, 8:\penalty0 e46331, July 2019.
\newblock ISSN 2050-084X.
\newblock \doi{10.7554/eLife.46331}.
\newblock URL \url{https://doi.org/10.7554/eLife.46331}.
\newblock Publisher: eLife Sciences Publications, Ltd.

\bibitem[Fan et~al.(2024)Fan, Doi, Gold, and Ding]{fan_neural_2024}
Yunshu Fan, Takahiro Doi, Joshua~I. Gold, and Long Ding.
\newblock Neural {Representations} of {Post}-{Decision} {Accuracy} and {Reward} {Expectation} in the {Caudate} {Nucleus} and {Frontal} {Eye} {Field}.
\newblock \emph{The Journal of Neuroscience}, 44\penalty0 (2):\penalty0 e0902232023, January 2024.
\newblock ISSN 0270-6474, 1529-2401.
\newblock \doi{10.1523/JNEUROSCI.0902-23.2023}.
\newblock URL \url{https://www.jneurosci.org/lookup/doi/10.1523/JNEUROSCI.0902-23.2023}.

\bibitem[Sugrue et~al.(2004)Sugrue, Corrado, and Newsome]{sugrue_matching_2004}
Leo~P. Sugrue, Greg~S. Corrado, and William~T. Newsome.
\newblock Matching {Behavior} and the {Representation} of {Value} in the {Parietal} {Cortex}.
\newblock \emph{Science}, 304\penalty0 (5678):\penalty0 1782--1787, June 2004.
\newblock ISSN 0036-8075, 1095-9203.
\newblock \doi{10.1126/science.1094765}.
\newblock URL \url{https://www.science.org/doi/10.1126/science.1094765}.

\bibitem[Dutilh et~al.(2012)Dutilh, van Ravenzwaaij, Nieuwenhuis, van~der Maas, Forstmann, and Wagenmakers]{dutilh_how_2012}
Gilles Dutilh, Don van Ravenzwaaij, Sander Nieuwenhuis, Han L.~J. van~der Maas, Birte~U. Forstmann, and Eric-Jan Wagenmakers.
\newblock How to measure post-error slowing: {A} confound and a simple solution.
\newblock \emph{Journal of Mathematical Psychology}, 56\penalty0 (3):\penalty0 208--216, June 2012.
\newblock ISSN 0022-2496.
\newblock \doi{10.1016/j.jmp.2012.04.001}.
\newblock URL \url{https://www.sciencedirect.com/science/article/pii/S0022249612000454}.

\bibitem[Rabbitt and Rodgers(1977)]{rabbitt_what_1977}
Patrick Rabbitt and Bryan Rodgers.
\newblock What does a {Man} do after he {Makes} an {Error}? {An} {Analysis} of {Response} {Programming}.
\newblock \emph{Quarterly Journal of Experimental Psychology}, 29\penalty0 (4):\penalty0 727--743, November 1977.
\newblock ISSN 0033-555X.
\newblock \doi{10.1080/14640747708400645}.
\newblock URL \url{http://journals.sagepub.com/doi/10.1080/14640747708400645}.

\bibitem[Bordelon et~al.(2020)Bordelon, Canatar, and Pehlevan]{Bordelon2020}
Blake Bordelon, Abdulkadir Canatar, and Cengiz Pehlevan.
\newblock Spectrum dependent learning curves in kernel regression and wide neural networks.
\newblock \emph{ArXiv e-prints}, 2020.
\newblock URL \url{https://arxiv.org/abs/2002.02561}.
\newblock tex.eprint: 2002.02561.

\bibitem[Mei et~al.(2022)Mei, Misiakiewicz, and Montanari]{Mei22Generalizationerrorrandom}
Song Mei, Theodor Misiakiewicz, and Andrea Montanari.
\newblock Generalization error of random feature and kernel methods: hypercontractivity and kernel matrix concentration.
\newblock \emph{Applied and Computational Harmonic Analysis}, 59:\penalty0 3--84, 2022.
\newblock Publisher: Elsevier tex.creationdate: 2022-07-20T21:54:34 tex.modificationdate: 2022-07-20T21:54:42.

\bibitem[Mingard et~al.(2023)Mingard, Rees, Valle-Pérez, and Louis]{mingard_deep_2023}
Chris Mingard, Henry Rees, Guillermo Valle-Pérez, and Ard~A. Louis.
\newblock Do deep neural networks have an inbuilt {Occam}'s razor?, April 2023.
\newblock URL \url{http://arxiv.org/abs/2304.06670}.
\newblock arXiv:2304.06670 [cs, stat].

\bibitem[{Li Deng}(2012)]{li_deng_mnist_2012}
{Li Deng}.
\newblock The {MNIST} {Database} of {Handwritten} {Digit} {Images} for {Machine} {Learning} {Research} [{Best} of the {Web}].
\newblock \emph{IEEE Signal Processing Magazine}, 29\penalty0 (6):\penalty0 141--142, November 2012.
\newblock ISSN 1053-5888.
\newblock \doi{10.1109/MSP.2012.2211477}.
\newblock URL \url{http://ieeexplore.ieee.org/document/6296535/}.

\bibitem[Glorot and Bengio(2010)]{glorot_understanding_2010}
Xavier Glorot and Yoshua Bengio.
\newblock Understanding the difficulty of training deep feedforward neural networks.
\newblock In \emph{Proceedings of the {Thirteenth} {International} {Conference} on {Artificial} {Intelligence} and {Statistics}}, pages 249--256. JMLR Workshop and Conference Proceedings, March 2010.
\newblock URL \url{https://proceedings.mlr.press/v9/glorot10a.html}.
\newblock ISSN: 1938-7228.

\bibitem[Krizhevsky()]{krizhevsky_learning_nodate}
Alex Krizhevsky.
\newblock Learning {Multiple} {Layers} of {Features} from {Tiny} {Images}.

\bibitem[Liu et~al.(2015)Liu, Luo, Wang, and Tang]{liu_deep_2015}
Ziwei Liu, Ping Luo, Xiaogang Wang, and Xiaoou Tang.
\newblock Deep {Learning} {Face} {Attributes} in the {Wild}, September 2015.
\newblock URL \url{http://arxiv.org/abs/1411.7766}.
\newblock arXiv:1411.7766 [cs].

\bibitem[Perron(1907)]{perron_zur_1907}
Oskar Perron.
\newblock Zur {Theorie} der {Matrices}.
\newblock \emph{Mathematische Annalen}, 64\penalty0 (2):\penalty0 248--263, June 1907.
\newblock ISSN 0025-5831, 1432-1807.
\newblock \doi{10.1007/BF01449896}.
\newblock URL \url{http://link.springer.com/10.1007/BF01449896}.

\bibitem[Jacot et~al.(2018)Jacot, Gabriel, and Hongler]{jacot_neural_2018}
Arthur Jacot, Franck Gabriel, and Clément Hongler.
\newblock Neural tangent kernel: {Convergence} and generalization in neural networks.
\newblock \emph{ArXiv e-prints}, June 2018.
\newblock URL \url{https://arxiv.org/abs/1806.07572}.
\newblock tex.eprint: 1806.07572.

\bibitem[Roberts et~al.()Roberts, Yaida, and Hanin]{roberts_principles_nodate}
Daniel~A Roberts, Sho Yaida, and Boris Hanin.
\newblock The {Principles} of {Deep} {Learning} {Theory}.
\newblock page 449.

\bibitem[Hecke(1917)]{Hecke17Uberorthogonalinvariante}
Erich Hecke.
\newblock Über orthogonal-invariante integralgleichungen.
\newblock \emph{Mathematische Annalen}, 78\penalty0 (1):\penalty0 398--404, 1917.
\newblock Publisher: Springer-Verlag tex.creationdate: 2022-07-23T15:10:14 tex.modificationdate: 2022-07-23T15:10:14.

\bibitem[Dutordoir et~al.(2020)Dutordoir, Durrande, and Hensman]{dutordoir_sparse_2020}
Vincent Dutordoir, Nicolas Durrande, and James Hensman.
\newblock Sparse {Gaussian} processes with spherical harmonic features.
\newblock In \emph{International {Conference} on {Machine} {Learning}}, pages 2793--2802. PMLR, 2020.
\newblock ISBN 2640-3498.

\bibitem[Erzan and Tuncer(2020)]{erzan_explicit_2020}
Ayşe Erzan and Aslı Tuncer.
\newblock Explicit construction of the eigenvectors and eigenvalues of the graph {Laplacian} on the {Cayley} tree.
\newblock \emph{Linear Algebra and its Applications}, 586:\penalty0 111--129, February 2020.
\newblock ISSN 0024-3795.
\newblock \doi{10.1016/j.laa.2019.10.023}.
\newblock URL \url{https://www.sciencedirect.com/science/article/pii/S002437951930463X}.

\bibitem[Brouwer and Haemers(2011)]{brouwer_spectra_2011}
Andries~E. Brouwer and Willem~H. Haemers.
\newblock \emph{Spectra of {Graphs}}.
\newblock Springer Science \& Business Media, December 2011.
\newblock ISBN 978-1-4614-1939-6.
\newblock Google-Books-ID: F98THwYgrXYC.

\bibitem[Feldman(2020)]{feldman_does_2020}
Vitaly Feldman.
\newblock Does learning require memorization? a short tale about a long tail.
\newblock In \emph{Proceedings of the 52nd {Annual} {ACM} {SIGACT} {Symposium} on {Theory} of {Computing}}, pages 954--959, Chicago IL USA, June 2020. ACM.
\newblock ISBN 978-1-4503-6979-4.
\newblock \doi{10.1145/3357713.3384290}.
\newblock URL \url{https://dl.acm.org/doi/10.1145/3357713.3384290}.

\bibitem[Van~Horn and Perona(2017)]{van_horn_devil_2017}
Grant Van~Horn and Pietro Perona.
\newblock The {Devil} is in the {Tails}: {Fine}-grained {Classification} in the {Wild}, September 2017.
\newblock URL \url{http://arxiv.org/abs/1709.01450}.
\newblock arXiv:1709.01450 [cs].

\bibitem[{Haibo He} and Garcia(2009)]{haibo_he_learning_2009}
{Haibo He} and E.A. Garcia.
\newblock Learning from {Imbalanced} {Data}.
\newblock \emph{IEEE Transactions on Knowledge and Data Engineering}, 21\penalty0 (9):\penalty0 1263--1284, September 2009.
\newblock ISSN 1041-4347.
\newblock \doi{10.1109/TKDE.2008.239}.
\newblock URL \url{http://ieeexplore.ieee.org/document/5128907/}.

\bibitem[Huang et~al.(2016)Huang, Li, Loy, and Tang]{huang_learning_2016}
Chen Huang, Yining Li, Chen~Change Loy, and Xiaoou Tang.
\newblock Learning {Deep} {Representation} for {Imbalanced} {Classification}.
\newblock In \emph{2016 {IEEE} {Conference} on {Computer} {Vision} and {Pattern} {Recognition} ({CVPR})}, pages 5375--5384, Las Vegas, NV, USA, June 2016. IEEE.
\newblock ISBN 978-1-4673-8851-1.
\newblock \doi{10.1109/CVPR.2016.580}.
\newblock URL \url{https://ieeexplore.ieee.org/document/7780949/}.

\bibitem[Ye et~al.(2021)Ye, Zhan, and Chao]{ye_procrustean_2021}
Han-Jia Ye, De-Chuan Zhan, and Wei-Lun Chao.
\newblock Procrustean {Training} for {Imbalanced} {Deep} {Learning}.
\newblock In \emph{2021 {IEEE}/{CVF} {International} {Conference} on {Computer} {Vision} ({ICCV})}, pages 92--102, Montreal, QC, Canada, October 2021. IEEE.
\newblock ISBN 978-1-66542-812-5.
\newblock \doi{10.1109/ICCV48922.2021.00016}.
\newblock URL \url{https://ieeexplore.ieee.org/document/9710650/}.

\bibitem[Francazi et~al.(2023)Francazi, Baity-Jesi, and Lucchi]{francazi_theoretical_2023}
Emanuele Francazi, Marco Baity-Jesi, and Aurelien Lucchi.
\newblock A {Theoretical} {Analysis} of the {Learning} {Dynamics} under {Class} {Imbalance}, June 2023.
\newblock URL \url{http://arxiv.org/abs/2207.00391}.
\newblock arXiv:2207.00391 [cs, stat].

\end{thebibliography}
\clearpage
\appendix

\section{Appendix / supplemental material}

\subsection{Overview}
\label{app:roadmap}
Our appendix has the following sections:
\begin{itemize}

\item In \cref{app:human-experiment}, we describe the human experiment in more detail and show the results of a replication in a second cohort. We furthermore report statistical tests and describe ethical considerations.

\item In \cref{app:discretization}, we outline how we bring neural network and human responses displayed in \cref{sec:learners} into a common space for direct comparison.

\item In \cref{app:continous-TPR-and-TNR-Rates}, we outline how we compute the TNR rates \(f^{tn}\) used in \cref{sec:early-dyn-ocs} and \cref{sec:learners}.

\item In \cref{app:theory}, we provide additional theoretical derivations and remaining proofs to the statements in the main text.

\item In \cref{app:shallow-OCS}, we show OCS signatures in \textit{shallow} networks with bias terms.

\item In the short \cref{app:TNR-TPR-no-bias}, we show how linear networks without bias terms behave on the task in \cref{sec:learners}.

\item In \cref{app:CNN-experiments}, we describe hyperparameters, datasets, and further training details used for our CNN experiments.

\item In  \cref{app:additional experiments}, we describe the results of additional experiments investigating early emergence of the OCS in non-linear models.

\item In \cref{app:linear-net-classimbalance} we show an additional solvable case of linear networks with bias terms under class imbalance.

\end{itemize}

\subsection{Human learning experiment}
\label{app:human-experiment}

We directly translated the hierarchical task setup used by \citet{saxe2019mathematical} into an experimental paradigm. Our design attempts to stay as close to the original task structure used for neural networks as possible. We designed the task as a mapping from 8 distinct input stimuli represented as planets to a set of 3 associated output stimuli represented as plants (see Fig. \ref{fig:human_task_supplement}, left). 

\begin{figure}[h]
    \includegraphics[width=0.9\linewidth]{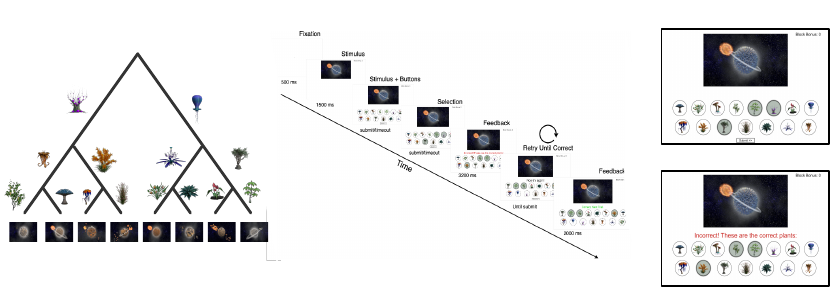}
    \centering
    \caption{Human task design. \textit{Left:} Hierarchical learning task, adapted for human participants. \textit{Centre:} Trial structure as experienced by human participants. \textit{Right:} Example screen during response period (top), Example screen during feedback period (bottom).}
    \label{fig:human_task_supplement}
\end{figure}

In the task, participants had to learn to associate which outputs properties are associated with each input. Unbeknownst to the participants we imposed a hierarchical structure on output targets (Fig. \ref{fig:human_task_supplement}, left). In the structure some output labels are associated with more than one input. As a control for analyses we also included an additional control input-output pair (similarly represented by a planet and a plant; not shown here and excluded from current analysis). We recruited a cohort of 10 subjects that were trained over the course of three days with one daily session. The cohort was recruited as part of a larger neuroimaging experiment but our analysis presented here is exclusive to behavioural results. We further replicated our results in a second cohort of 46 human subjects recruited via the online platform Prolific (prolific.com). Results of the replication of the study can be seen in \cref{fig:human_replication}.

\begin{figure}[h]
    \includegraphics[width=.7\linewidth]{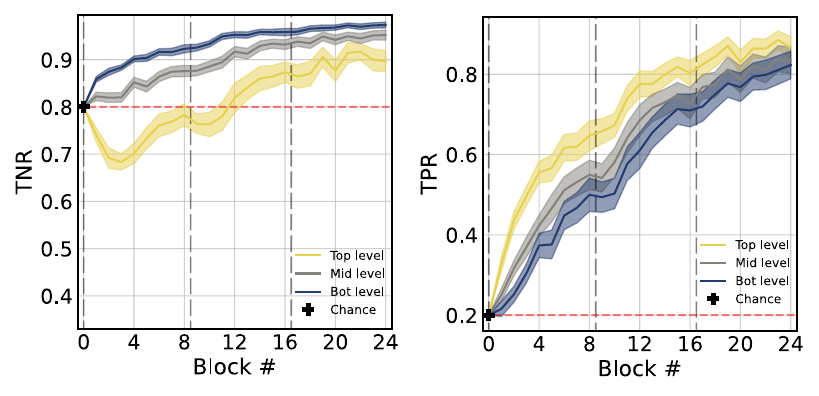}
    \centering
    \caption{The human replication cohort. While learning is slower, the qualitative pattern indicating reliance on the OCS is replicated. \textit{Left}: TNR rate. \textit{Right:} TPR rate.}
    \label{fig:human_replication}
\end{figure}

Each day of at home training consisted of 8 blocks of training with 22 trials each (160 standard trials and 16 control trials) which lasted about one hour. The trial structure during training is shown in \cref{fig:human_task_supplement}, centre.  During training trials, subjects were shown the stimulus on screen and were required to press three buttons, presented below the planet image (\cref{fig:human_task_supplement}, right). The subjects received fully informative feedback on each trial and were forced to repeat the trial in the case of incorrect selection until the correct properties were selected. The location of buttons was shuffled on screen for each trial and for each forced repetition. For each button clicked correctly on their first attempt participants received a bonus point. We displayed a block-wise bonus in the corner of the screen throughout the task. Participants were payed slightly above local minimum wage as a baseline and received a substantial performance dependent bonus (on average about one-third of the baseline pay). We include a screenshot of the initial instructions in \cref{fig:instructions}. Beyond this initial instruction screen participants received more nuanced instructions about clicking of buttons and feedback in the beginning of the task.

\textbf{Statistical tests.} While our focus is on qualitative patterns in human behaviour, we compute statistical tests on the TNR rates for human results seen in the main text (\cref{fig:Human-net-correspondence}). We averaged all blocks in a given day and performed a two-way repeated measures ANOVA to assess the effect of day and hierarchy level on TNR. The two-way repeated measures ANOVA revealed significant main effects of day \(F (2, 18) = 57.22,\: p < .0001, \:\eta^2 = .25 \text{ and level } F (2, 18) = 6.25,\: p = .033, \:\eta^2 = .18\). Beyond this we also found a significant interaction of day and level \(F (4, 36) = 9.795,\: p = .0056 ,\: \eta^2 = .042\). A Mauchly test indicated that the assumption of sphericity had been violated for level \(\chi^2 = .03,\: p < .5 \text{ and the interaction term } \chi^2 = .006,\: p < .5\). Significance values are reported with Greenhouse-Geisser correction. The results confirm that performance between levels are significantly different depending on day and hierarchical level. 

\textbf{Ethical considerations.} Human participants performed a simple, computerised learning task without the collection of personal identifiable information or substantial deception. Human data collection was handled strictly in line with institutional guidelines and under institutional review board approval. We obtained informed consent for each participant before commencing the study. We highlighted that participants could withdraw at any time without penalty or loss of compensation by simply exiting full-screen or informing the experimenter. We provided contact emails in the case of concern or questions. Data was handled in a strictly anonymised format and stored on password secured devices. Participants were payed above minimum wage for their country of origin.

\begin{figure}[h]
    \includegraphics[width=.9\linewidth]{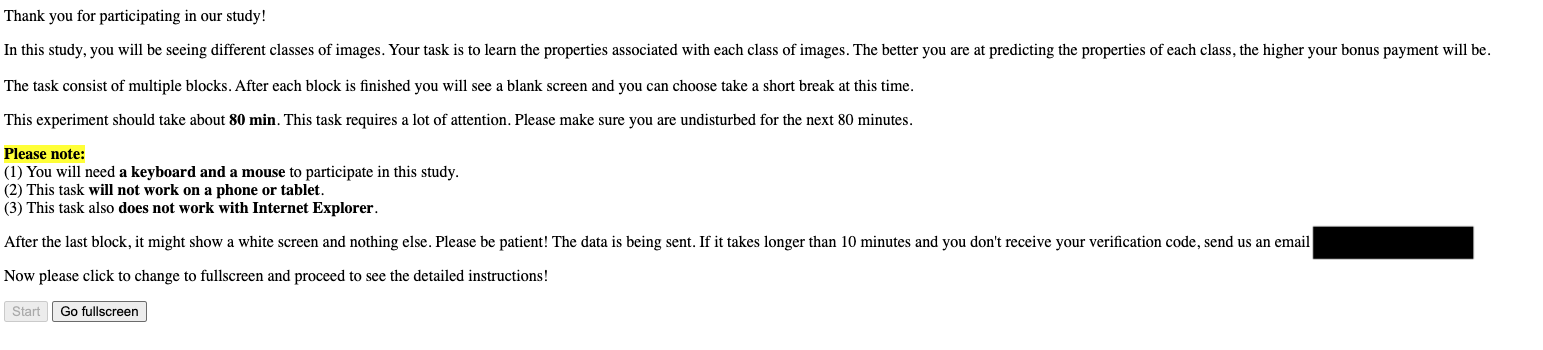}
    \centering
    \caption{Initial instructions received by participants after the collection of informed consent.}
    \label{fig:instructions}
\end{figure}

\subsection{Discretizing network responses for comparison to humans}
\label{app:discretization}

\textbf{Discretization.} In our task, neural networks produced continuous outputs. This is distinct from human learners who were required to give discrete responses. We now describe the discretization that allows us to compare human and network responses. Fundamentally, we conceptualise inference as a noisy process by which responses are sampled from a distribution over output labels. That is, we treat outputs from our linear network as logits. We first feed network outputs \(\mathbf{\hat{y}}_i\) through a softmax function with temperature 0.2 and subsequently sample three responses without replacement. The procedure maps continuous outputs \(\mathbf{\hat{y}}_i\) to binary responses vectors in \(\{0,1\}^{N_{out}}\). 

\textbf{Expected solutions.} Here we describe the derivation of expected solutions used in \cref{fig:Human-net-correspondence}, dashed lines. The derivation of these "expected responses" under the sampling procedure allows to make the reliance of network responses on the exact solutions in \cref{sec:early-dyn-bias} clear.

Consider network outputs \(\hat{\mathbf{y}}_i(t) = \mathbf{W}^2(t)\mathbf{W}^1(t)\mathbf{x}_i\). We transform these outputs through a softmax function \(\sigma_{\beta}: \mathbb{R}^{N_{out}} \rightarrow (0,1)^{N_{out}}\). Let \(S = \{s_1, s_2, s_3\}\) denote the set of three unique response indices sampled from \(\sigma_{\beta}(\hat{\mathbf{y}}_i(t))\) without replacement, where \(s_n \in \{1, 2, \ldots, N_{out}\}\) for \(n = 1, 2, 3\), and all \(s_n\) are hence distinct. The probability distribution \(\sigma_{\beta}(\hat{\mathbf{y}}_i(t))\) is dependent on time \(t\), therefore denote the produced probability of \(S\) as \(P_t(S)\). For each of these sets \(S\) we can compute an associated TNR for each of the \(k\in\{1,2,3\}\) levels in the hierarchy. We denote this random variable as \(X^{k}_{S}\). We can then compute expected solutions to inference behaviour as
\begin{equation}
\mathbb{E}_t[X^{k}_{S}] = \sum_{S \subseteq \{1, 2, ..., m\},\,|S| = 3} P_t(S) X^{k}_{S}
\end{equation}

\subsection{TNR rate/Correct-rejection score}
\label{app:continous-TPR-and-TNR-Rates}
Here we describe the metric used in the bottom panel of \cref{fig:outputs} and in \cref{fig:Human-net-correspondence}. The metrics effectively describes TNR (correct-rejection scores). We use the metric on continuous network responses in \(\mathbb{R}^{N_{out}}\) in \cref{fig:outputs}. We also use the metric on discretised networks responses in \(\{0,1\}^{N_{out}}\) and for human responses in \(\{0,1\}^{N_{out}}\) in \cref{fig:Human-net-correspondence}.

Given responses in \(\mathbf{\hat{y}}\) or in and target vectors \(\mathbf{y} \in \mathbb{R}^{N_{out}}\) the metric computes the alignment between target and response vectors while only focusing on zero entries in \(\mathbf{y}\). Furthermore we compute the metric separately for the \(k \in \{1,2,3\} \) separate levels of the hierarchy where the entries \(s\) and \(e\) denote relevant start and  end indices of level \(k\) in the vectors \(\mathbf{\hat{y}}\) and \(\mathbf{y}\). The metric is then computed as

\begin{equation}
\label{cont_TNR}
f^{tn}_k(\mathbf{\hat{y}}, \mathbf{y}) = \frac{(\mathbf{1}_{N_{out}} -\mathbf{\hat{y}})_{s:e}^T\,(\mathbf{1}_{N_{out}} - \mathbf{y})_{s:e}}{\sum^{e}_{i=s}(\mathbf{1}_{N_{out}} - \mathbf{y})_i},
\end{equation}
where $s:e$ is a "slicing" notation that takes the subvector between indices $s$ and $e$.

If for all desired entries of \(0\) in \(\mathbf{y}\) the vector \(\mathbf{\hat{y}} \) is equal to \(0\) the metric will be at 1. Correspondingly if entries in \(\mathbf{\hat{y}} \) are larger than zero the metric \(f^{tn}_k(\mathbf{\hat{y}}, \mathbf{y})\) will decrease. Thus, the metric measures wrong beliefs about the presence of target labels across the different levels of the hierarchy.

\subsection{Additional theoretical results and proofs}
\label{app:theory}

\subsubsection{Equivalence of bias terms}\label{app:Equivalence-of-bias}
In this section, we give more detail on the method used in in \cref{sec:early-dyn-bias} of how to reformulate a bias term in terms of the network weights and a constant feature in the input.

Consider a network with an explicit input bias term $\mathbf{b}^{1}$,
\[
\hat{\mathbf{y}}=\tilde{\mathbf{W}}^{1}\tilde{\mathbf{x}}+\tilde{\mathbf{b}}^{1}
\]

This is equivalent to introducing a constant component to the vector
$\mathbf{x}$, 
\[
\tilde{\mathbf{x}}\rightarrow \mathbf{x}\coloneqq\left[\begin{array}{c}
1\\
\tilde{\mathbf{x}}
\end{array}\right],
\]

and using the network
\[
\hat{\mathbf{y}}=\mathbf{W}^{1}\mathbf{x},
\]

as we can write
\begin{align*}
\left(\mathbf{W}^\mathbf{{1}}\mathbf{x}\right)_{m} & =\sum_{j=0}^{N_{in}}W_{mj}^{1}x_{j}\\
 & =W_{m0}^{1}1+\sum_{j=1}^{N_{in}}W_{mj}^{1}x_{j}\\
 & =W_{m0}^{1}1+\sum_{j=0}^{N_{in}-1}\tilde{W}_{mj}^{1}\tilde{x}_{j}\\
 & \equiv b_{m}^{1}+\sum_{j=0}^{N_{in}-1}\tilde{W}_{mj}^{1}\tilde{x}_{j}.
\end{align*}

In order to match a given i.i.d. initialization $b_{m}^{1}\sim\mathcal{N}\left(0,\,\sigma_{b}^{2}\right)$
where $\sigma_{b}\neq\sigma_{w}$, the component that needs to be
added to $\tilde{\mathbf{x}}$ to get equivalence needs to be $\sigma_{b}/\sigma_{w}$.

\subsubsection{Learning dynamics for bias terms}\label{app:learning-dynamics-bias-terms}

We here derive analytical expressions for the learning speeds of input
and output bias terms for a two-layer deep linear network discussed in the
main text, 

\[
\hat{\mathbf{y}}=\mathbf{W}^{2}\left(\mathbf{W}^{1}\mathbf{x}+\mathbf{b}^{1}\right)+\mathbf{b}^{2}.
\]

We decompose $\mathbf{W}^{2}=\mathbf{UA}^{(2)}\mathbf{R}^{(2)}$ and $\mathbf{W}^{1}=\mathbf{R}^{(1)}\mathbf{A}^{(1)}\mathbf{V}$
by means of a singular value decomposition (SVD). We here make the assumption of balancedness
$\mathbf{W}^{1}(0)\mathbf{W}^{1T}(0)=\mathbf{W}^{2T}(0)\mathbf{W}^{2}(0)$ \citep{braun_exact_2022} at the beginning
of training, which implies $\mathbf{R}^{(2)}\mathbf{S}^{(2)2}\mathbf{R}^{(2)T}=\mathbf{R}^{(1)}\mathbf{S}^{(1)2}\mathbf{R}^{(1)T}$.
For clarity, we further assume the simplification
\[
\mathbf{R}^{(2)T}=\mathbf{R}^{(1)}\eqqcolon \mathbf{R},\:\mathbf{A}^{(2)}=\mathbf{A}^{(1)}\eqqcolon\sqrt{\mathbf{A}}.
\]
We here just state these relations without further comment to complement the respective derivation for the weights in \citep{saxe_exact_2014}. This decomposition then allows to rewrite the gradients. 

\paragraph*{Input bias term}

\begin{align*}
\tau\dt\mathbf{b}^{1} & =\nabla_{\mathbf{b}^{1}}\mathcal{L}\\
 & =\left(\mathbf{y}-\hat{\mathbf{y}}\right)^{T}\mathbf{W}^{2}\\
 & =\left(\mathbf{y}-\left(\mathbf{W}^{2}\left(\mathbf{W}^{1}\mathbf{x}+\mathbf{b}^{1}\right)+\mathbf{b}^{2}\right)\right)^{T}\mathbf{W}^{2}\\
\mathbb{E}_{\mathbf{x}} & \rightarrow\left(\bar{\mathbf{y}}-\mathbf{W}^{2}\left(\mathbf{W}^{1}\bar{\mathbf{x}}+\mathbf{b}^{1}\right)-\mathbf{b}^{2}\right)^{T}\mathbf{W}^{2}\\
 & =\left(\bar{\mathbf{y}}-\mathbf{UAV}\bar{\mathbf{x}}-\mathbf{U\sqrt{\mathbf{A}}\mathbf{R}b}^{1}-\mathbf{b}^{2}\right)^{T}\mathbf{U}\sqrt{\mathbf{A}}\mathbf{R}^{T}\\
 & =\bar{\mathbf{y}}^{T}\mathbf{U}\sqrt{\mathbf{A}}\mathbf{R}^{T}-\bar{\mathbf{x}}^{T}\mathbf{V}^{T}\mathbf{A}\mathbf{R}^{T}-\mathbf{b}^{1T}\mathbf{R}\mathbf{A}\mathbf{R}-\mathbf{b}^{2T}\mathbf{U}\sqrt{\mathbf{A}}\mathbf{R}^{T}\\
 & =\left(\mathbf{Y}\1\right)^{T}\mathbf{U}\sqrt{\mathbf{A}}\mathbf{R}^{T}-\left(\mathbf{X}\1\right)^{T}\mathbf{V}^{T}\mathbf{A}\mathbf{R}-\mathbf{b}^{1T}\mathbf{R}\mathbf{A}\mathbf{R}^{T}-\mathbf{b}^{2T}\mathbf{U}\sqrt{\mathbf{A}}\mathbf{R}^{T}.
\end{align*}

Here, we denoted the expectation over the data samples as $\mathbb{E}_{\mathbf{x}}$. Projecting from the right with $\mathbf{R}_{\alpha}\in\mathbb{R}^{N_{\text{hidden}}}$
gives

\begin{equation}
\tau\dt\left(\mathbf{b}^{1T}\mathbf{R}_{\alpha}\right)=\bar{\mathbf{y}}^{T}\mathbf{U}_{\alpha}\sqrt{a_{\alpha}}-\bar{\mathbf{x}}^{T}\mathbf{V}_{\alpha}^{T}a_{\alpha}-\mathbf{b}^{1T}\mathbf{R}_{\alpha}a_{\alpha}-\mathbf{b}^{2T}\mathbf{U}_{\alpha}\sqrt{a_{\alpha}}.\label{eq:input-bias}
\end{equation}

\paragraph*{Output bias term}

\begin{align}
\tau\dt\mathbf{b}^{2} & =\left(\mathbf{y}-\left(\mathbf{W}^{2}\left(\mathbf{W}^{1}\mathbf{x}+\mathbf{b}^{1}\right)+\mathbf{b}^{2}\right)\right)\nonumber \\
\mathbb{E}_{\mathbf{x}} & \rightarrow\bar{\mathbf{y}}-\mathbf{W}^{2}\left(\mathbf{W}^{1}\bar{\mathbf{x}}+\mathbf{b}^{1}\right)-\mathbf{b}^{2}\nonumber \\
 & =\bar{\mathbf{y}}-\mathbf{UAV}\bar{\mathbf{x}}-\mathbf{U}\sqrt{\mathbf{A}}\mathbf{R^{T}b}^{1}-\mathbf{b}^{2}\nonumber \\
 & =\mathbf{Y}\1-\mathbf{UAVX}\1-\mathbf{UR^{T}b}^{1}-\mathbf{b}^{2}.\label{eq:output_bias}
\end{align}

Notably, the derivative in \cref{eq:input-bias} is proportional to
the singular vectors of the weights $a_{\alpha}$, so that its growth
is attenuated, analogous to the sigmoidal growth in deep linear networks \citep{saxe_exact_2014}. In contrast, the learning signal $\dt \mathbf{b}^{2}$ in \cref{eq:output_bias}
is not affected by the initialization of the weights and is hence
$\mathcal{O}(1)$ already at the beginning of learning, reminiscent of shallow networks. 

\subsubsection{Neural tangent kernel}\label{app:NTK}

In this section, we review the neural tangent kernel (NTK). This
object is useful because it directly describes the learning dynamics
in output space $\hat{y}$ \citep{jacot_neural_2018,roberts_principles_nodate} as we briefly
demonstrate here. We then calculate the NTK for our specific architecture
to yield \cref{eq:NTK} in the main text. The following makes use of Einstein summation
convention. 

For a vector-valued model $\hat{\mathbf{y}}(\mathbf{x})\in\mathbb{R}^{N_{out}}$ parametrized
by a parameter vector $\theta$, the evaluation on sample $\mathbf{x}_{i}$
from training data at $\mathbf{x}_{i'}$ evolves as
\begin{align}
\tau\dt y_{m}(\mathbf{x}_{i}) & =\sum_{k}\frac{dy_{m}(\mathbf{x}_{i})}{d\theta^{k}}\frac{d\theta^{k}}{dt}\label{eq:def_NTK}\\
 & =-\eta\sum_{k}\frac{dy_{m}(\mathbf{x}_{i})}{d\theta^{k}}\frac{d\mathcal{L}}{d\theta^{k}}\\
 & =-\eta\left[\sum_{k}\frac{dy_{m}(\mathbf{x}_{i})}{d\theta^{k}}\frac{dy_{m'}(\mathbf{x}_{i'})}{d\theta^{k}}\right]\frac{d\mathcal{L}}{dy_{m'}}(\mathbf{x}_{i'})\\
 & \eqqcolon-\eta\,\NTK_{mm'}(\mathbf{x}_{i},\,\mathbf{x}_{i'})\left(y_{m'}(\mathbf{x}_{i'})-\hat{y}_{m'}(\mathbf{x}_{i'})\right),
\end{align}

where we used the chain rule and that the parameters update according
to gradient descent with learning rate $\eta$, $\frac{d\theta^{k}}{dt}=-\eta\frac{d\mathcal{L}}{d\theta^{k}}$. The last line has defined the NTK. 
We set $\eta=1$ in the main text for simplicity, as it does not change
trajectory and thereby convergence in the case of gradient flow. In
addition, we evaluated $\frac{d\mathcal{L}}{dy_{m'}}(\mathbf{x}_{i'})$ for
the case of MSE loss $\mathcal{L}(\mathbf{x}_{i'})=\nicefrac{1}{2}\sum_{m'}\left(y_{m'}(\mathbf{x}_{i'})-\hat{{y}}_{m'}(\mathbf{x}_{i'})\right)^{2}$.
The last line of \cref{eq:def_NTK} reveals that the NTK acts as an
effective learning rate, as noted by \citet{roberts_principles_nodate}.

We here consider a two-layer linear architecture $\hat{Y}_{m}^{i}(\mathbf{X})=W_{mk}^{2}\left(W_{kj}^{1}X_{j}^{i}+b_{k}^{1}\right)+b_{m}^{2}$
where we adopt Einstein summation convention over repeated indices.
The parameters are $\theta^{k}\in\left\{ \mathbf{W}^{2},\mathbf{W}^{2},\mathbf{b}^{1},\mathbf{b}^{2}\right\} $.
Herein, $m$ indexes output features and $i$ indexes data samples.
The non-zero gradients are 
\begin{align*}
\frac{d\hat{Y}_{m}^{i}}{dW_{mk}^{2}} & =W_{kj}^{1}X_{j}^{i}+b_{k}^{1}\\
\frac{d\hat{Y}_{m}^{i}}{db_{m}^{2}} & =1_{m}\\
\frac{d\hat{Y}_{m}^{i}}{dW_{kj}^{1}} & =W_{mk}^{2}X_{j}^{i}\\
\frac{d\hat{Y}_{m}^{i}}{db_{k}^{1}} & =W_{mk}^{2}1_{k}.
\end{align*}

Inserting this into \cref{eq:def_NTK}, we get
\begin{align*}
\mathsf{NTK}_{m_{1}m_{2}}(X_{\;j}^{i_{1}},X_{j}^{\,i_{2}}) & =I_{m_{1}m_{2}}\:\left(X_{\;j'}^{i_{1}}W_{j'k}^{1}W_{kj''}^{1}X_{j''}^{\,i_{2}}+b_{k}^{1}b_{k}^{1}\right)\\
 & +1_{m_{1}}1_{m_{2}}\\
 & +W_{m_{1}k}^{2}W_{km_{2}}^{2T}\:X_{\;j}^{i_{1}}X_{j}^{\,i_{2}}\\
 & +W_{m_{1}k}^{2}1_{k}1_{k}W_{km_{2}}^{2}.
\end{align*}
or in matrix notation, collecting similar terms
\begin{align*}
\mathsf{NTK}(\mathbf{X},\mathbf{X}) & =\mathbf{I}_{N_{out}}\otimes \mathbf{X}^{T}\mathbf{W}^{1T}\mathbf{W}^{1}\mathbf{X}+\mathbf{b}^{1T}\mathbf{b}^{1}\\
 & +\mathbf{1}\mathbf{1}^{T}\,\otimes\,\underbrace{\mathbf{1}\mathbf{1}^{T}}_{\leftrightarrow \mathbf{b}^{2}}\\
 & +\mathbf{W}^{2}\mathbf{W}^{2T}\,\otimes\,\Bigl(\mathbf{X}^{T}\mathbf{X}+\underbrace{\mathbf{11}^{T}}_{\leftrightarrow b^{1}}\Bigr)\\
 & \in\mathbb{R}^{N_{out}\times N_{out}}\,\otimes\,\mathbb{R}^{N\times N},
\end{align*}

where the left hand side operator in the tensor product $\otimes$
is acting in output space $m_{1}m_{2}$, whereas the right hand side
operator acts in pattern space $i_{1}i_{2}$. The notation $\leftrightarrow \mathbf{b}$
indicates that a term is due to the bias term. To illustrate this, the NTK acts on the set of labels $\mathbf{Y}\in\mathbb{R}^{N_{in}\times N}$ as follows:
\begin{equation}
    \left(\NTK(\mathbf{X},\mathbf{X})\,\mathbf{Y}\right)_{i}^{m}=\sum_{m'}^{N_{out}}\sum_{i'}^{N}\NTK(X_{i},X_{i'})^{mm'}Y_{i'}^{m'}.
\end{equation}
For simplicity, we approximate $\mathbf{W}^{2}(0)\mathbf{W}^{2T}(0)=\sigma_{\mathbf{W}}^{2}I_{N_{out}}$
and $\mathbf{W}^{1T}(0)\mathbf{W}^{1}(0)=\sigma_{\mathbf{W}}^{2}I_{N_{in}}$, which approximately holds for
initialization
\[
\mathbf{W}^{1}(0)\sim\mathcal{N}\left(0,\,\sigma_{\mathbf{W}}^{2}/N_{hid}\right),\,\mathbf{W}^{2}(0)\sim\mathcal{N}\left(0,\,\sigma_{\mathbf{W}}^{2}/N_{hid}\right),\,\mathbf{b}^{1}=0,\,\mathbf{b}^{2}=0 
\]
where $N_{hid}$ is the size of the hidden layer and both $N_{in}$ and $N_{hid}$ are large. This leaves

\begin{align*}
\mathsf{NTK}(\mathbf{X},\mathbf{X}) & =\sigma_{\mathbf{W}}^{2}\mathbf{I}_{N_{out}}\,\otimes\,\Bigl(2\mathbf{X}^{T}\mathbf{X}+\underbrace{\mathbf{11}^{T}}_{\leftrightarrow b^{1}}\Bigr)\;+\;\mathbf{11}^{T}\,\otimes\,\underbrace{\mathbf{11}^{T}}_{\leftrightarrow \mathbf{b}^{2}}.
\end{align*}

\subsection*{Proofs}

\subsubsection{Feasibility of closed-form solution}

\commutemain*

\begin{proof} \label{app:commute}

We would like to know when the right singular vectors $\mathbf{V}$ (denote as $\mathbf{V}^{yx}$ here for clarity)
of $\mathbf{\Sigma}^{yx}=\mathbf{U}^{yx}\mathbf{S}^{yx}\mathbf{V}^{yx}$
match these of $\mathbf{\Sigma}^{x}=\mathbf{U}^{x}\mathbf{S}^{x}\mathbf{V}^{x}$.
First, to reduce the problem to $\mathbf{V}^{yx}$, note
that $\mathbf{\Sigma}^{yxT}\mathbf{\Sigma}^{yx}=\mathbf{V}^{yx}\mathbf{S}^{yx2}\mathbf{V}^{yx}$,
so that what remains to show is $\left[\mathbf{\Sigma}^{yxT}\mathbf{\Sigma}^{yx},\,\mathbf{\Sigma}^{xx}\right]=0$,
where $\left[\mathbf{A},\,\mathbf{B}\right]\coloneqq\mathbf{AB}-\mathbf{BA}$
denotes the commutator between two matrices $\mathbf{A}$ and $\mathbf{B}$.
We compute the two terms as 
\begin{align*}
\mathbf{\Sigma}^{yxT}\mathbf{\Sigma}^{yx}\mathbf{\Sigma}^{xx} & =\mathbf{X}\mathbf{Y}^{T}\mathbf{Y}\mathbf{X}^{T}\mathbf{X}\mathbf{X}^{T}\\
\mathbf{\Sigma}^{xxT}\mathbf{\Sigma}^{yxT}\mathbf{\Sigma}^{yx} & =\mathbf{XX}^{T}\mathbf{XY}^{T}\mathbf{YX}^{T}
\end{align*}

The commutator vanishes if these terms match, which happens for the
simpler equality 
\[
\mathbf{Y}^{T}\mathbf{YX}^{T}\mathbf{X}=\mathbf{X}^{T}\mathbf{XY}^{T}\mathbf{Y},
\]

or $\left[\mathbf{Y}^{T}\mathbf{Y},\,\mathbf{X}^{T}\mathbf{X}\right]=0$. The converse follows only if
the tranformation $\mathbf{X}\ldots \mathbf{X}^{T}$ in the former equation is invertible,
which is the case if a left inverse $\mathbf{X}^{-1}\mathbf{X}=\mathbf{I}_{N_{in}}$ exists.

\end{proof}

\subsubsection{OCS and shared properties correspond to each other}
We here link the OCS and shared properties stand in close relation, as the eigenvector $\1$ represents properties that are shared across all data samples. 
\ocsonemain*

\begin{proof} \label{app:ocsone} 
\[
\mathbf{XX}^{T}\bar{\mathbf{x}}=\left(\mathbf{XX}^{T}\right)\frac{1}{N}\mathbf{X}\1=\frac{1}{N}\mathbf{X}\left(\mathbf{X}^{T}\mathbf{X}\right)\1=\frac{1}{N}\mathbf{X}\lambda\,\1=\lambda\frac{1}{N}\mathbf{X}\1=\lambda\bar{\mathbf{x}}.
\]
\end{proof} 

\subsubsection{Constant data mode $\1$ is related to symmetry in the data}

This section gives proof sketches based on symmetry in the dataset that
are sufficient to make $\1$ an eigenvector to $\mathbf{X}^{T}\mathbf{X}$ and $\mathbf{Y}^{T}\mathbf{Y}$,
and in particular hold for the dataset that we are considering. We anticipate that it is possible to formulate these statements in a more universal way by fully leveraging the cited literature. 

The assumptions on symmetry should intuitively at least hold in an approximate manner for
many datasets, we expect that they indeed are the reason why we observe
a prevalence of $\1$, although they are not a necessary condition.

\paragraph*{Continuously supported data $\mathbf{x}\in\mathbb{R}^{N_{in}}$}
\label{app:one_in_XX_YY_general}
\begin{restatable}[Continuous symmetry induces $\1$]{proposition}{onecontinuous}
If the pairwise correlations $\mathbf{y}_{i}^{T}\mathbf{y}_{i'}$
in a dataset are rotationally symmetric, its similarity matrix $\mathbf{Y}^{T}\mathbf{Y}$
has eigenvector $\1$. Note that this is a weaker assumption than
the data itself being symmetric.

\end{restatable}

\begin{proof}

We assume that $\mathbf{Y},\mathbf{X}$ have been sampled from a ground truth data distribution
$p(\mathbf{y},\mathbf{x})$. If $p(\mathbf{y})$ is rotationally symmetric and
$\mathbf{X}$ is comprised of samples $\mathbf{x}$ that are uniformly distributed
on the hypersphere, we can introduce the kernel function $\mathbf{y}_{\mathbf{x}_{i}}^{T}\mathbf{y}_{\mathbf{x}_{i'}}=k(\mathbf{x}_{i},\,\mathbf{x}{}_{i'})=k(\mathbf{R}_{l}\mathbf{x}_{i},\,\mathbf{R}_{l}\mathbf{x}_{i'})=k(\mathbf{x}_{i}^{T}\mathbf{x}_{i'})$
for any $R_{l}$ that is a representation of the group of rotations
$G=\text{SO}(N_{in})$ that faithfully acts on the ``subsampled''
hypersphere $\mathbf{X}$ comprised of vectors $\mathbf{x}\in\mathbb{R}^{N_{in}}$.
It therefore only depends on the pairwise input similarity (hence
sometimes called dot-product kernel). If follows that for all vectors
$\mathbf{v}(\mathbf{X})\in\mathbb{R}^{N}$ that are evaluations of the functions
of the sample points $\mathbf{X}$

\[
\mathbf{Y}^{T}\mathbf{Y}\,\mathbf{v}=k(\mathbf{X}^{T}\mathbf{X})\,\mathbf{v}=k(\left(\mathbf{R}_{l}\mathbf{X}\right)^{T}\left(\mathbf{R}_{l}\mathbf{X}\right))\,\mathbf{v}=\mathbf{R}_{k}^{T}k(\mathbf{X}^{T}\mathbf{X})\mathbf{R}_{l}\,\mathbf{v}\:\Leftrightarrow\:\left[\mathbf{Y}^{T}\mathbf{Y},\,\mathbf{R}_{l}\right]=0,
\]

where $\left[\mathbf{A},\mathbf{B}\right]\eqqcolon\mathbf{AB}-\mathbf{BA}$ is the commutator between two matrices.

It follows that we must have for all rotations $\mathbf{R}_{l}$

\[
\mathbf{R}_{l}\left(\mathbf{Y}^{T}\mathbf{Y}\,\1\right)=\mathbf{Y}^{T}\mathbf{YR}_{l}\,\1=\mathbf{Y}^{T}\mathbf{Y}\,\lambda_{\mathbf{R}_{l}}\1=\lambda_{\mathbf{R}_{l}}\mathbf{Y}^{T}\mathbf{Y}\,\1
\]

with eigenvalue $\lambda_{\mathbf{R}_{l}}=1$.

meaning that $\mathbf{Y}^{T}\mathbf{Y}\,\1$ is an eigenvector to \emph{all} $\mathbf{R}_{l}$.
This can only be the case if $\mathbf{Y}^{T}\mathbf{Y}\,\1\propto\1$, as this is the
only vector of values on the sphere that is invariant under any rotations. 

\end{proof}

We point out that it can be shown more generally with tools from functional
analysis that the full spectrum of this kernel operator $k$ are the
spherical harmonics if the data measure $p(\mathbf{x})$ is spherically
symmetric \citep{Hecke17Uberorthogonalinvariante}, see \citep{dutordoir_sparse_2020} for a modern
presentation with tools from calculus. As the first harmonic $\mathcal{Y}_{l=0,m=0}(\mathbf{x})$
is constant, it follows that also the constant function $\mathbf{1}(\mathbf{x})\equiv1$
is an eigenfunction when drawing a finite set of samples from this
kernel.

\paragraph{Data on a graph $\mathbf{x}\in\mathbb{R}^{N_{in}}$}

We here prove that the former statement holds for the hierarchical
dataset that is discussed in the main text, i.e. that $\1$ is an
eigenvector to $\mathbf{Y}^{T}\mathbf{Y}$.

First, note that it is easy to convince oneself of this by writing down the
matrices explicitly: Then, as the rows are just permutations of one
another, $\1$ is immediately identified as an eigenvector, because
$\sum_{i'}^{N}\mathbf{Y}_{i}^{T}\mathbf{Y}_{i'}{1}_{i'}=\mathbf{Y}_{i}^{T}\bigl(\sum_{i'}\mathbf{Y}_{i'}\bigr)$
will then not depend on $i$ and hence be proportional to $\1$. 

To connect with the former symmetry-based argument \cref{app:one_in_XX_YY_general},
we here however give a proof that is based on the symmetry in the data: 
\label{app:hierarchical_one} 
\begin{restatable}[Discrete symmetry induces $\1$]{proposition}{hierarchical_one}

Consider a connected Cayley tree graph with adjacency matrix $\mathbf{A}$
and nodes $\mathbf{x}_{i}$. Furthermore, let $\mathbf{R}_{l}\in G$ be an element
of a faithful representation of the symmetry group $G$ that acts
on the graph nodes $\mathbf{v}$, i.e. that leaves its adjacency matrix invariant,  $\left[\mathbf{R}_{l},\mathbf{A}\right]=0\:\forall\,\mathbf{R}_{l}$.

If $\mathbf{Y}$ are labels associated with the leaf nodes $\mathbf{X}$ (the outermost
generation of the graph, see \citep{erzan_explicit_2020}) and there exists a
similarity function $k$ such that $\mathbf{y}_{\mathbf{x}_{i}}^{T}\mathbf{y}_{\mathbf{x}_{i'}}=\mathbf{y}_{\mathbf{R}_{l}\mathbf{x}_{i}}^{T}\mathbf{y}_{\mathbf{R}_{l}\mathbf{x}_{i'}}=k(d(\mathbf{x}_{i},\mathbf{x}_{i'}))$
$\forall\,\mathbf{R}_{l}$ where $d$ is the geodesic distance on the graph,
$\1$ will be an eigenvector of $\mathbf{Y}^{T}\mathbf{Y}$. 

\end{restatable}

\begin{proof}

From the symmetry assumption on the labels, we again have for any
vector $\mathbf{v}$ of node loadings $\left[\mathbf{R}_{l},\,\mathbf{Y}^{T}\mathbf{Y}\right]\mathbf{v}=0\:\forall\,\mathbf{R}_{l}\in G$.
From this, we find that 
\[
\mathbf{R}_{l}\left(\mathbf{Y}^{T}\mathbf{Y}\,\1\right)=\mathbf{Y}^{T}\mathbf{Y}\,\mathbf{R}_{l}\1=\mathbf{Y}^{T}\mathbf{Y}\,\lambda_{\mathbf{R}_{l}}\1=\lambda_{\mathbf{R}_{l}}\mathbf{Y}^{T}\mathbf{Y}\,\1\:\forall\,\mathbf{R}_{l}
\]

This shows that $\mathbf{Y}^{T}\mathbf{Y}\,\1$ is an eigenvector of $\mathbf{R}_{l}$ with eigenvalue
$\lambda_{\mathbf{R}_{l}}=1$ for any element of the symmetry group. The only
vector $\mathbf{v}$ that is invariant under \emph{all} symmetry operations
of the graph is the constant vector $\1$.

\end{proof}

We briefly point out the rich literature on spectral graph theory
(for example \citep{brouwer_spectra_2011,erzan_explicit_2020}) that might allow making
statements about the nature of the eigenvalues and other eigenvectors
as a function of the graph topology. We expect that this is possible
because the literature in the continuous case discussed in the next paragraph bases their arguments
on the Laplacian on the sphere, an operator that can be extended to graphs as well.
We leave these exploration for future work. 

\begin{corollary}

Because $k(\mathbf{X}^{T}\mathbf{X})\coloneqq \mathbf{X}^{T}\mathbf{X}$ defines a particular case of
input-output similarity mapping, $\1$ is also an eigenvector to
$\mathbf{X}^{T}\mathbf{X}$ under the former assumptions of uniform data distribution.

\end{corollary}

\subsubsection{Constant data mode $\1$ is the leading eigenvector}\label{app:leading_ev}

Here, we prove that the constant eigenvector $\1$ which is responsible
for the OCS solution is associated with the \textit{leading} eigenvalue
of the input-output correlation matrix and hence drives early learning.

\leadingmain*

\begin{proof}

Let $\1$ be an eigenvector to both similarity matrices $\mathbf{X}^{T}\mathbf{X}$
and $\mathbf{Y}^{T}\mathbf{Y}$ associated with eigenvalue $\tilde{\lambda}$. Moreover,
let $\mathbf{X}^{T}\mathbf{X}$ and $\mathbf{Y}^{T}\mathbf{Y}$ have positive entries. Then, the Perron-Frobenius
theorem \citep{perron_zur_1907} guarantees that $\tilde{\lambda}$
is indeed the leading eigenvector to $\mathbf{Y}^{T}\mathbf{Y}$, $\tilde{\lambda}\equiv\lambda_{0}=s_{0}^{2}$.

By \cref{thm:one_and_avg}, $\bar{\mathbf{x}}$ and $\bar{\mathbf{y}}$ are now also the leading
eigenvectors for $\mathbf{XX}^{T}$ and $\mathbf{YY}^{T}$. Because the
eigenvectors of $\mathbf{YY}^{T}$ and $\mathbf{XX}^{T}$ are the left and right singular
vectors of $\mathbf{\Sigma}^{yx}$, respectively, with the eigenvalues being
the squares of the singular values, it follows that 
\[
s_{0}\mathbf{u}_{0}\mathbf{v}_{0}^{T}=\sqrt{\lambda_{0}}\bar{\mathbf{y}}\bar{\mathbf{x}}^{T}.
\]

\end{proof}

\subsection{Shallow network OCS learning}
\label{app:shallow-OCS}

In this brief section we show the OCS signatures of shallow networks with bias terms. The result is displayed in \cref{fig:shallow-net-bias}. We see similar behavioural signatures to deep linear networks. However, the tendency to the OCS is more transient.

\begin{figure}[h]
    \includegraphics[width=0.8\linewidth]{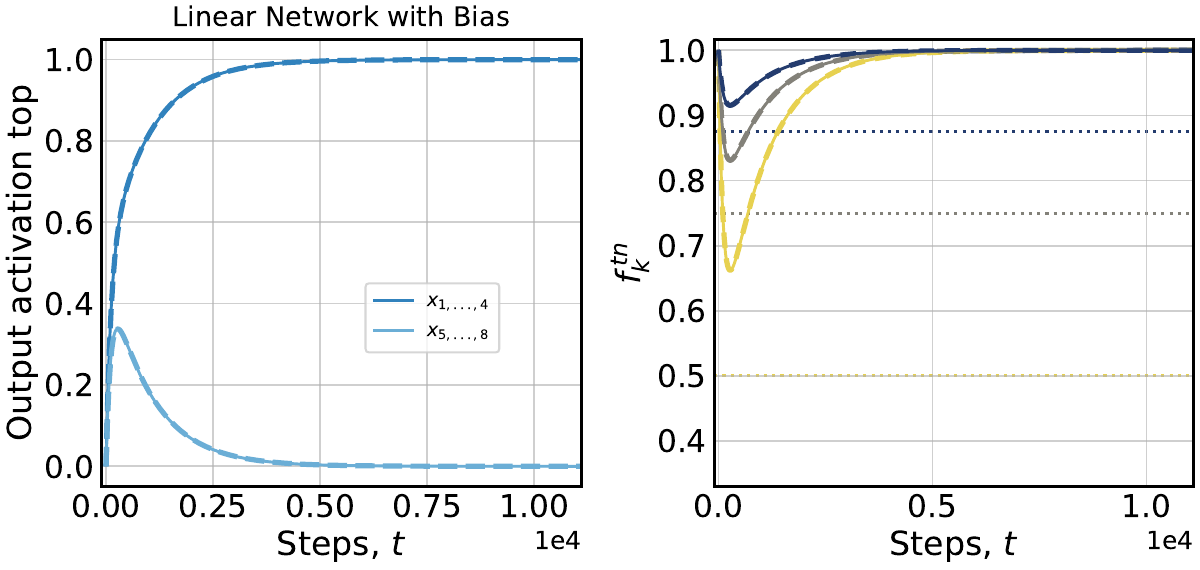}
    \centering
    \caption{Early learning in shallow networks with bias terms approaches the OCS.}
    \label{fig:shallow-net-bias}
\end{figure}

\subsection{TNR in linear networks without bias terms}
\label{app:TNR-TPR-no-bias}
In this short section we provide a supplemental figure relevant for our results in \cref{sec:learners}: We train deep and shallow linear networks \textit{without} bias terms. The learning setting and computation of metrics are equivalent to results in \cref{fig:Human-net-correspondence}. We display the result in \cref{fig:linear_nets_no_bias_TNR}. While networks learn the task, early, response biases are fully absent in these models.

\begin{figure}[h]
    \includegraphics[width=.7\linewidth]{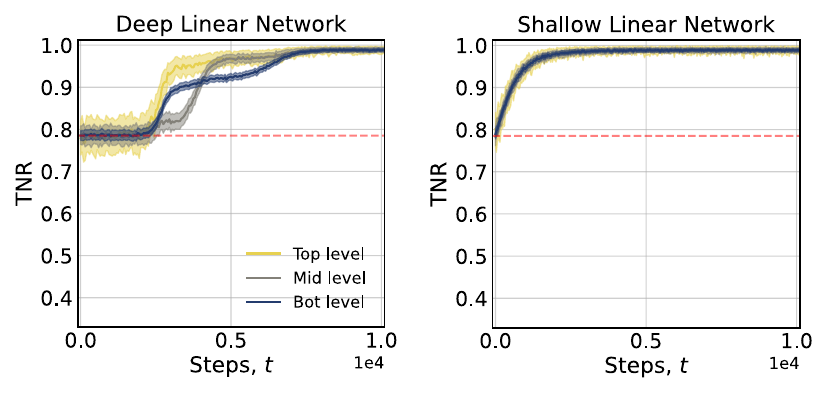}
    \centering
    \caption{True negative rates for linear networks \textit{without} bias terms. We do not see characteristic response patterns observed in \cref{fig:Human-net-correspondence}.}
    \label{fig:linear_nets_no_bias_TNR}
\end{figure}

\subsection{CNN datasets and hyperparameters}
\label{app:CNN-experiments}
\textbf{Datasets used. } We used and adapted different image datasets for our experiments with CNNs. While the main text focused on results obtained with a variant of MNIST we report further experiments we conducted to highlight the universality of early OCS learning. 

\begin{enumerate}
    \item \textbf{Hierarchical MNIST.} We randomly sampled eight digit classes. We subsequently replaced standard one-hot labels with label vectors containing the hierarchical structure as seen in \cref{fig:Task-and-setting}.
    
    \item \textbf{Hieararchical CIFAR-10.} We applied the same procedure and randomly sampled eight classes from CIFAR-10 and replaced one-hot labels as for the hierarchical MNIST.
    
    \item \textbf{Imbalanced-binary-MNIST.} While not described in the main text we also report results in a setting with standard one-hot target vectors. We randomly sample two MNIST classes for training. To assess the impact of the OCS in early learning we introduced class imbalanced by oversampling one of the two classes by a factor of two.

    \item \textbf{Standard CelebA.} We perform experiments on CelebA's face attribute detection task. The task offers a natural testbed for early learning of the OCS as face attribute target labels form a non-uniform distribution as seen in \cref{fig:celeba-OCS}, bottom. We also normalised images in the dataset before training.

\end{enumerate}

\textbf{Model details.} 
We trained a custom CNN with 3 convolutional layers (layer 1: 32 filters of size 5\(\times\)5; layer 2: 64 filters of size 3\(\times\)3; layer 3: 96 filters of size 3\(\times\)3), followed by 2 fully connected layers of sizes 512 and 256. Activation functions for all layers were chosen as ReLUs. The final layer of the model did not contain an activation function when training with squared error loss. In experiments with the class imbalanced-binary-MNIST and cross-entropy loss the final layer contained a softmax function as non-linearity. For experiments on CelebA the final layer contained sigmoid activation functions and we trained with a binary cross-entropy loss over all 40 labels.

\textbf{Training details.} 
For our results on hierarchical MNIST we train models with minibatch SGD with a batch size of 16 and with a relatively small step size of 1e-4 to examine the early learning phase. For all experiments we used Xavier uniform initialisation \citep{glorot_understanding_2010}. Whenever we use bias terms in the model we initialize these as 0 in line with common practice. For our main experiments we train models using a simple squared error loss function. However, to demonstrate generality we repeat experiments for the case of class imbalance using a cross-entropy loss and binary cross-entropy in the case of CelebA. All experiments are repeated 10 times with different random seeds with the exception of CelebA where we used 5 different random seeds, we provide standard errors in all figures (shaded regions). For experiments on the hierarchical CIFAR-10, the class imbalanced MNIST, and CelebA we kept all parameters as above but we increase step size to 1e-3. We trained CNN models on an internal cluster on a single RTX 5000 GPU. Runs took less than one hour to complete.

\subsection{Additional experimental results}
\label{app:additional experiments}
To understand the generality of OCS learning we plot the results of experiments examining early learning of the OCS in these models. We mostly restrict ourselves to plots as seen in \cref{fig:MNIST vs Orthogonal MNIST} as we deem these figures most instructive.

\textbf{Hieararchical CIFAR-10.} We train on a hierarchical version of CIFAR-10. Where we randomly sample 8 classes from MNIST and replace target labels by hierarchical vectors as in \cref{sec:early-dyn-ocs}. We find the key signatures of early OCS learning: We find early indifference, the reversion of performance metrics to the OCS, and a small initial distance of average response the to the OCS solution.  The results mirror behaviour on the hierarchical MNIST shown in \cref{fig:outputs} and \cref{fig:MNIST vs Orthogonal MNIST}.

\begin{figure}[h]
    \includegraphics[width=.9\linewidth]{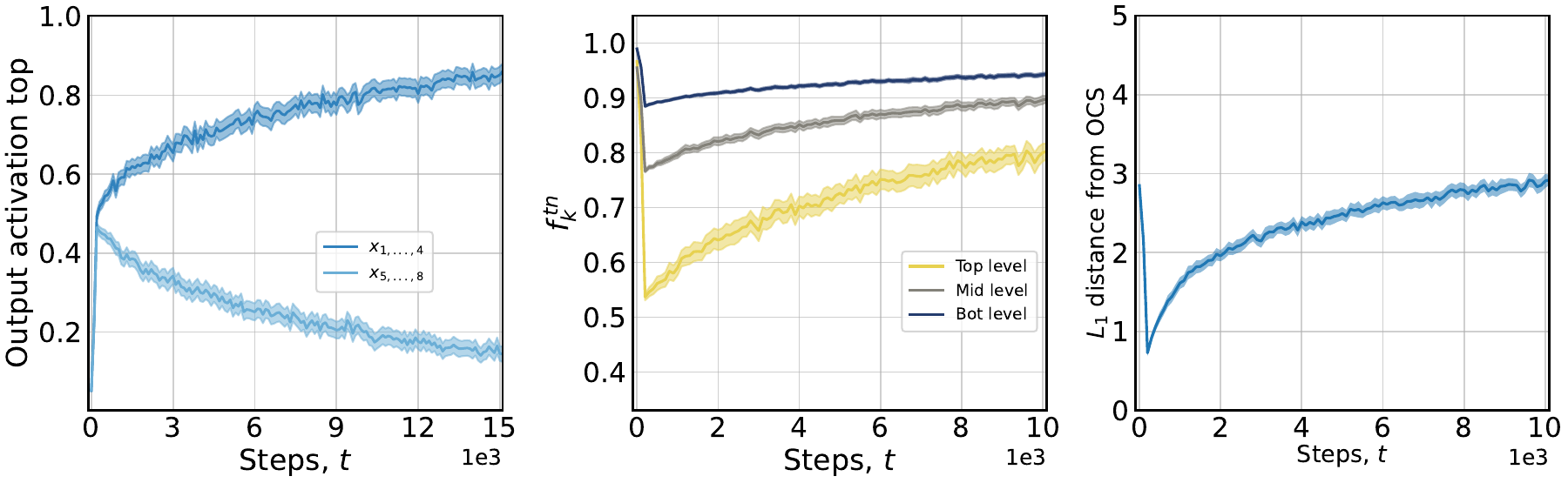}
    \centering
    \caption{Early OCS learning CNNs trained on hierarchical CIFAR-10. \textit{left:} Network outputs for a single output unit in response to all inputs \(\mathbf{x}_i\). \textit{Centre:} Performance metrics \(f^{tn}\) (\cref{app:continous-TPR-and-TNR-Rates}). \textit{Right:} Mean distance of network responses from OCS. Averages taken over every 10 batches for plotting.}
    \label{fig:Cifar-10-hieararchy}
\end{figure}

\textbf{Class-imbalance MNIST.} We train on an imbalanced MNIST task as described in \cref{app:CNN-experiments}. We plot the results for training with squared error and cross-entropy loss functions in \cref{fig:MNIST-Imbalanced}. Both settings show reversion to the OCS. Note that average model outputs in the case of the cross-entropy loss start relatively close to outputs expected under the OCS. Despite this proximity the model is still driven towards the OCS solution. The results on this imbalanced case highlight potential fairness implications. Given that network have been found to revert to the OCS when generalising \citep{kang_deep_2024}, early learning in the OCS setting can transiently, but significantly disadvantage minority classes. We further highlight this point in a second solvable case of linear networks with bias terms in \cref{app:linear-net-classimbalance}. 

\begin{figure}[h]
    \includegraphics[width=.7\linewidth]{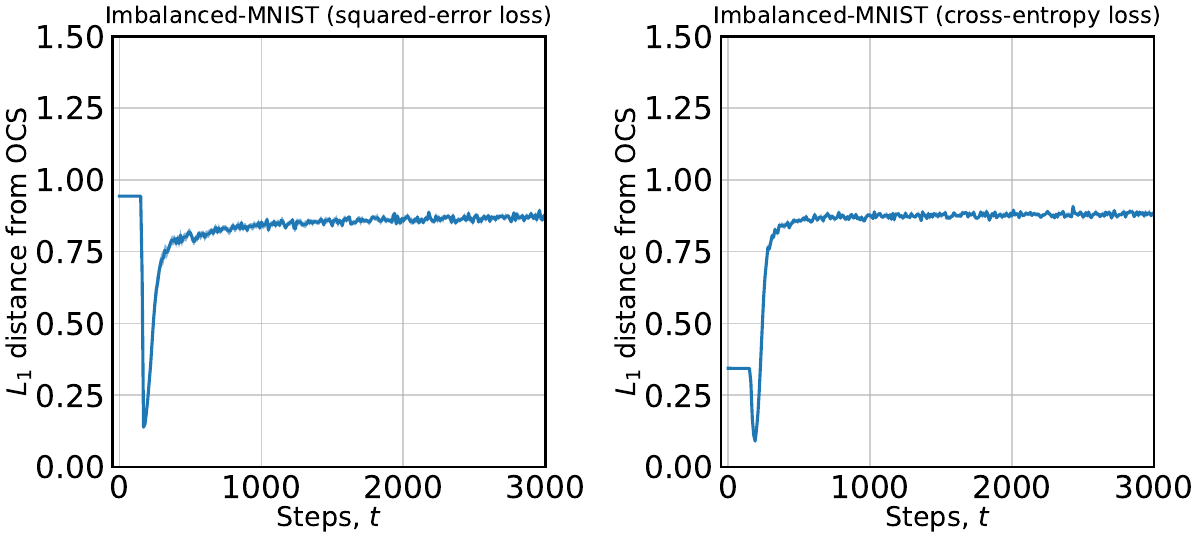}
    \centering
    \caption{Mean distance of network responses from OCS in CNNs trained on the imbalanced MNIST task. Averages taken over each batch.}
    \label{fig:MNIST-Imbalanced}
\end{figure}

\textbf{Standard CelebA.} We show distance from the OCS for the CelebA face attribute detection task in \cref{fig:celeba-OCS}, top. CelebA provides a useful test for our hypothesis as attribute labels display natural imbalances. We highlight the strong non-uniformity of the majority attribute labels in \cref{fig:celeba-OCS}, bottom. We again train networks in two variants: one with squared-error loss and one with binary cross-entropy loss applied over all 40 face attributes. With both loss functions network responses are driven towards the OCS in early learning. This case further highlights the universality of early OCS learning. OCS learning might be especially undesirable in this setting for fairness reasons as the model will be overly liberal in the prediction of the most common face attributes. 

\begin{figure}[h]
    \includegraphics[width=.7\linewidth]{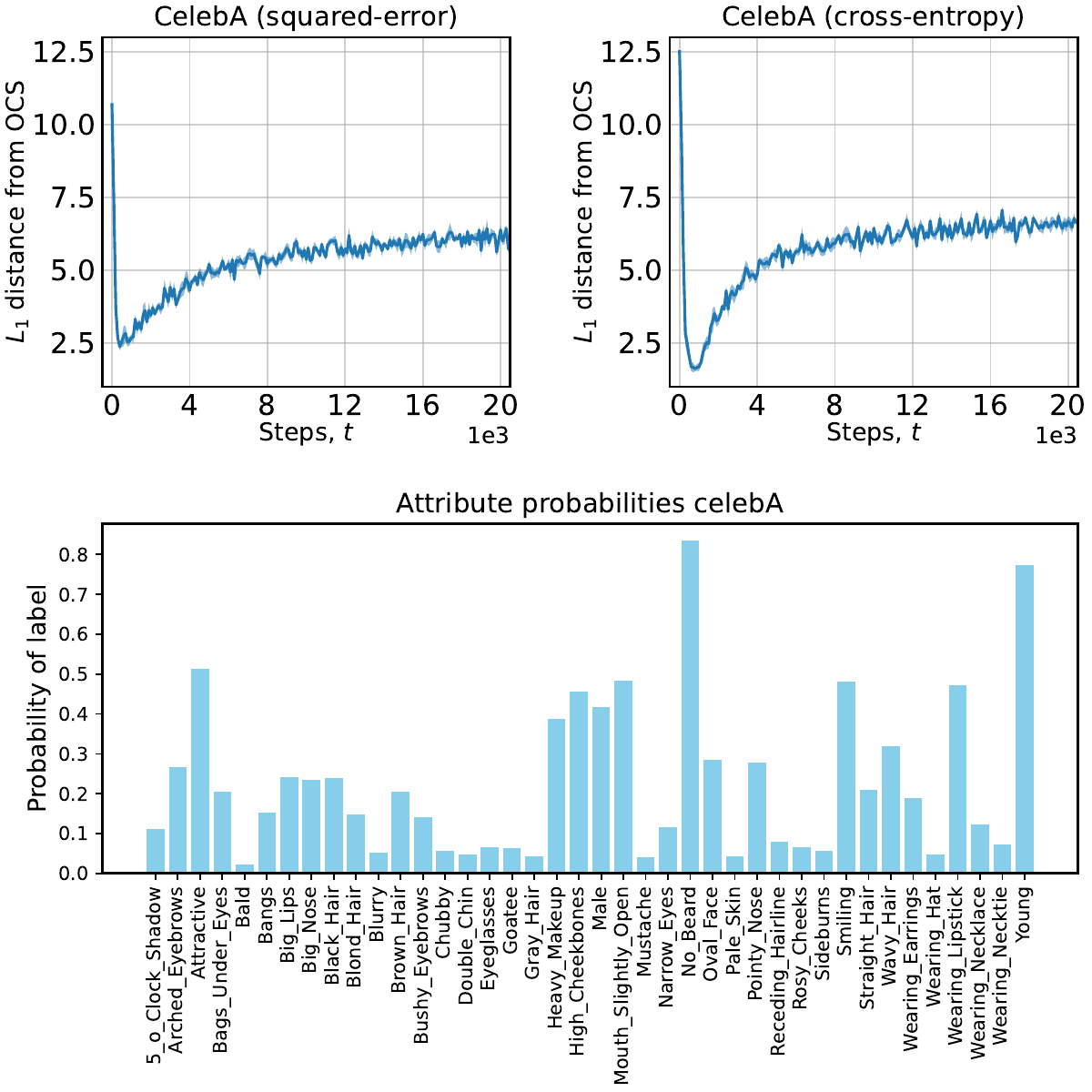}
    \centering
    \caption{\textit{Top:} Mean distance of network responses from the OCS in CNNs trained on the CelebA face attribute prediction task. \textit{Bottom:} Marginal probabilities of CelebA face attributes.}
    \label{fig:celeba-OCS}
\end{figure}

\subsection{Linear networks under class imbalance}
\label{app:linear-net-classimbalance}
In this section, we describe a second case of a solvable linear network with bias terms. Our dataset consists of two examples where one example appears twice as frequently. We show the data used on the right side of \cref{fig:OCS-class-imbalance}. The minority class has two identifying labels, while this construction appears artificial, it allows for the application of \cref{thm:commute} and solutions to learning dynamics from \cref{sec:lin-net-formalism} apply.

The case is of particular practical relevance as it illustrates the impact of early OCS learning under class imbalance, a common problem in machine learning where datasets are often naturally imbalanced \citep{feldman_does_2020, van_horn_devil_2017}. In practice, these settings are often addressed through oversampling of minority classes \citep{haibo_he_learning_2009, huang_learning_2016}. Empirical work by \citet{ye_procrustean_2021} documented that neural networks initially fail to learn information about the minority class while classifying most minority examples as belonging to the majority class. Subsequent theoretical work by \citet{francazi_theoretical_2023} demonstrated that the phenomenon is caused by competition between the optimisation of different classes. 

Our work adds to this literature by providing dynamics in a case of gradient-based learning under class imbalance learning that is exactly solvable.
Our exact solutions highlight the potential role of early OCS learning in the initial failure to learn about minority classes. The OCS solution substantially biases early predictions towards the majority class as seen in \cref{fig:OCS-class-imbalance}, centre. The results also can be understood as solvable analogous to early reversion to the OCS seen in the Imbalanced-binary-MNIST setting in \cref{fig:MNIST-Imbalanced}. 
The results highlight the potential fairness implications of early OCS learning as the learning phase systematically biases the model against the minority classes.

\begin{figure}[h]
    \includegraphics[width=1\linewidth]{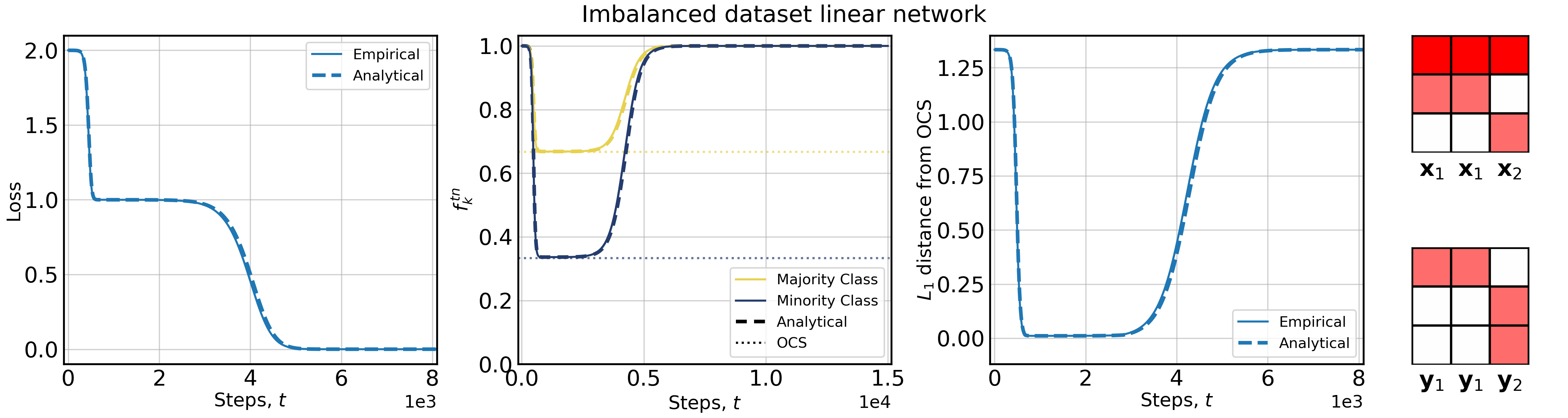}
    \centering
    \caption{Early learning of the OCS in linear networks under class imbalance.}
    \label{fig:OCS-class-imbalance}
\end{figure}

\end{document}